\documentclass{article} 
\usepackage{iclr2027_conference,times}

\usepackage{amsmath,amsfonts,bm}

\def\eqref#1{equation~\ref{#1}}

\def\1{\bm{1}}

\DeclareMathAlphabet{\mathsfit}{\encodingdefault}{\sfdefault}{m}{sl}
\SetMathAlphabet{\mathsfit}{bold}{\encodingdefault}{\sfdefault}{bx}{n}

\usepackage{hyperref}
\usepackage{url}
\usepackage{graphicx}
\usepackage{multirow}
\usepackage{subcaption}
\usepackage{booktabs}
\usepackage{wrapfig}
\usepackage{amsmath}
\usepackage{amssymb}
\usepackage{amsfonts}
\usepackage{xcolor}
\usepackage{colortbl}
\usepackage{tikz}    
\usepackage{cleveref}
\usepackage{enumitem}
\usepackage{setspace}
\usepackage{titlesec}
\usepackage{algorithm}
\usepackage{algorithmic}
\usepackage{amsthm}
\usepackage{microtype}
\newtheorem{assumption}{Assumption}
\newtheorem{proposition}{Proposition}
\usepackage{xspace}
\newcommand{\bpm}{\textsc{bpm}\xspace}
\usepackage{tablefootnote}

\title{The Missing Coefficients: Bayesian Pairwise Merging for Model Personalization}

\author{Yaling Shen, \,Tongtong Wu, \,Siyuan Yan, \,Gholamreza Haffari \\
Monash University\\
\texttt{\{yaling.shen, gholamreza.haffari\}@monash.edu} \\
}

\iclrfinalcopy 
\begin{document}

\maketitle
\fancyhead{} 
\lhead{Preprint}

\begin{abstract}
How can we personalize a shared expert library from a user's pairwise choices? 
Prior work can realize different reward trade-offs by merging reward-specialized experts, given a vector of trade-off weights. 
In practice, users can more naturally choose between outputs than specify numerical weights. 
The challenge is therefore to turn these choices into the coefficients required for merging, while accounting for ambiguity when feedback is limited. 
Our key idea is to treat the unknown reward weights as latent variables: infer a posterior over them from pairwise choices and reward-score differences, and use its mean directly as the merge coefficients. 
We instantiate this idea as Bayesian Pairwise Merging (\bpm), whose posterior also characterizes which reward trade-offs remain plausible given the feedback.
We evaluate \bpm on radiology summarization, image captioning, and story generation, spanning text-to-text and image-to-text generation.
With 100 feedback per simulated persona, \bpm achieves macro decided win rates of 91.7\%, 77.1\%, and 64.3\% against uniform merge.
For six pairs of simulated personas, each prefers the model fitted to its own feedback, a pattern also observed in a human proof-of-concept.
In simulations under \bpm's model and prior, its nominal $90\%$ intervals for temperature-scaled reward weights achieve task-averaged marginal coverage of $88.9\%$ and $89.2\%$ with only $10$ and $25$ comparisons, respectively.
\bpm thus enables personalization from pairwise feedback without per-user policy training, while characterizing the coefficient ambiguity left by limited feedback.
\end{abstract}

\section{Introduction}\label{sec:intro}

Personalizing generative models requires adapting to how users trade off competing objectives~\citep{guan2025survey,ji2026survey}. 
Existing multi-objective methods can realize different trade-offs once a preference vector is given~\citep{zhong2024panacea, lin2025parm}. 
A representative example is Rewarded Soups~\citep{rame2023rewardedsoups}, which merges reward-specialized experts with different coefficients to construct models for different preferences. 
The same expert library can then be reused across users, avoiding a separate training run for every trade-off. 
These methods assume that the preference weights are already known.
Users, however, may find it easier to compare generated outputs than to specify numerical weights over reward dimensions. 
The central question is \emph{how a user's pairwise choices can supply the missing coefficients for model personalization}.


Answering this question requires connecting preference inference with model construction.
Prior work has shown that pairwise choices can inform estimates of user-specific reward weights~\citep{bose2025lore}, but limited feedback may leave multiple weight vectors plausible~\citep{shenfeld2025pref}.
Yet even accurate reward-weight estimates do not guarantee the intended trade-offs in the merged model~\citep[\S2.2.3]{rame2023rewardedsoups}.
Thus, translating uncertain preference estimates into effective merge coefficients remains a challenge.


\begin{figure}[t]
\begin{center}
\includegraphics[width=\textwidth]{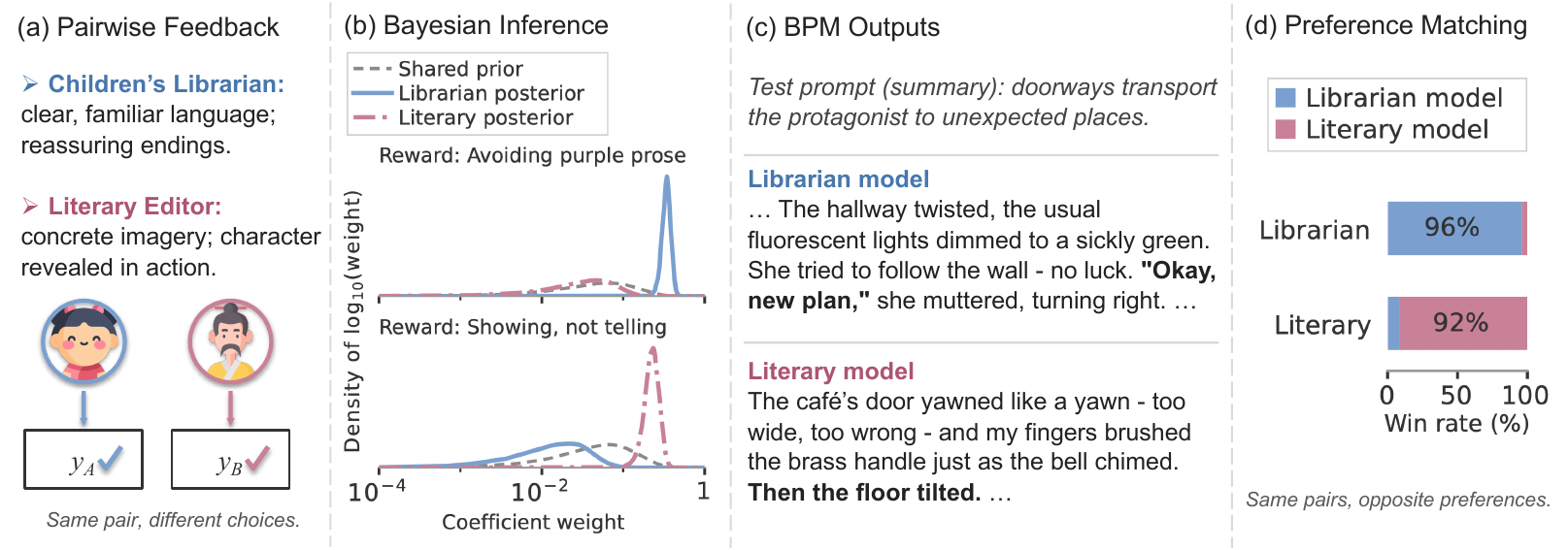}
\end{center}
\vspace{-1.5em}
\caption{\textbf{Pairwise feedback personalizes a merged model.}
Two simulated personas provide different choices~\textit{(a)}, inducing distinct posteriors~\textit{(b)} and merged-model outputs~\textit{(c)}.
Each persona favors its corresponding model on shared response pairs~\textit{(d)}.}
\vspace{-1.8em}
\label{fig:teaser}
\end{figure}

In this work, we propose \textbf{Bayesian Pairwise Merging (\bpm)}, which connects uncertain preference inference with the construction of a personalized model (\Cref{fig:teaser}).
To account for the ambiguity left by limited feedback, \bpm uses Bayesian inference to obtain a posterior over user-specific reward weights, quantifying the relative support for different trade-offs and the uncertainty that remains.
The posterior mean then supplies the coefficients for merging shared reward-specialized experts, enabling classical Bayesian inference to guide the personalization of modern deep generative models without per-user policy training.
Our theoretical analysis makes explicit the additional assumptions needed to connect accurate weight estimation with optimality.
Experiments on radiology summarization~\citep{zhang2018radsum}, image captioning~\citep{chen2015cococaptions}, and story generation~\citep{fan2018hierarchical} demonstrate effective personalization and characterize the stability and uncertainty of the inferred coefficients.

Our main contributions are threefold.
\begin{enumerate}[leftmargin=*,itemsep=0pt,topsep=0pt,partopsep=0pt]
\item \textbf{Pairwise feedback supplies effective model-construction coefficients.}
\bpm connects Bayesian reward-weight inference to shared-expert merging, constructing personalized models without explicit user-specified weights or per-user policy training.
The inferred coefficients improve on uniform merging across three
generation tasks, establishing a practical route from pairwise
choices to personalized models
(\S\ref{sec:method} and~\S\ref{sec:results-main}).

\item \textbf{Inferred coefficients translate preference differences into model behavior.}
Across three generation tasks, simulated personas favor models fitted to their own feedback, and controlled changes in feedback preferences produce corresponding behavioral shifts.
Together, these findings demonstrate preference following beyond aggregate improvements in model performance (\S\ref{sec:results-follow}).

\item \textbf{\bpm stabilizes coefficients and quantifies uncertainty.}
\bpm's posterior means vary less than \textsc{mle} coefficients across feedback subsets at $n\leq100$ in all three tasks.
In simulations under \bpm's model and prior, compact intervals for temperature-scaled reward weights achieve near-nominal coverage.
\bpm thus provides reduced coefficient sensitivity and quantified reward-weight ambiguity, complementing output preference as distinct dimensions of personalization (\S\ref{sec:analysis}, Appendix~\ref{app:analysis}).
\end{enumerate}
\section{Related Work}\label{sec:related_work}
\paragraph{Multi-objective alignment with specified preferences.}
Given a preference vector, existing methods implement the corresponding trade-offs through weighted-reward optimization~\citep{wu2023finegrained, zhou2024modpo}, preference-conditioned policies~\citep{yang2024ric, zhong2024panacea, wang2024clp, wang2024dpa}, decoding-time composition~\citep{shi2024mod}, or preference-aware model merging~\citep{rame2023rewardedsoups, xie2025bone, chen2025pareto, li2025map}.
Related work on multi-teacher on-policy distillation integrates specialized teachers into a student through additional training~\citep{ma2026mopd}.
\bpm addresses how to obtain the preference weights themselves from an individual's pairwise choices, then uses them to merge shared experts without per-user policy training.

\paragraph{Preference inference from pairwise feedback.}
Prior work learns reward functions from human comparisons~\citep{christiano2017deep} and uses Bayesian inference to estimate reward weights as well as their uncertainty~\citep{sadigh2018active, biyik2020arl}.
PReF~\citep{shenfeld2025pref} and LoRe~\citep{bose2025lore} infer user-specific weights over reward components learned across users, while other methods model preference distributions~\citep{poddar2024vpl, lee2026vrf}.
\bpm uses predefined reward dimensions without requiring other users' preference data.
In connecting feedback to personalized outputs, UserAlign selects responses from a fixed candidate pool~\citep{padurean2025useralign}, while Spectral Souping estimates policy-combination coefficients from online feedback, including through sequential Gaussian posterior updates~\citep{chow2026spectral}.
\bpm infers a posterior over reward weights from a fixed offline set of pairwise judgments.

\section{Bayesian Pairwise Merging}\label{sec:method}

\begin{figure}[t]
\begin{center}
\includegraphics[width=\textwidth]{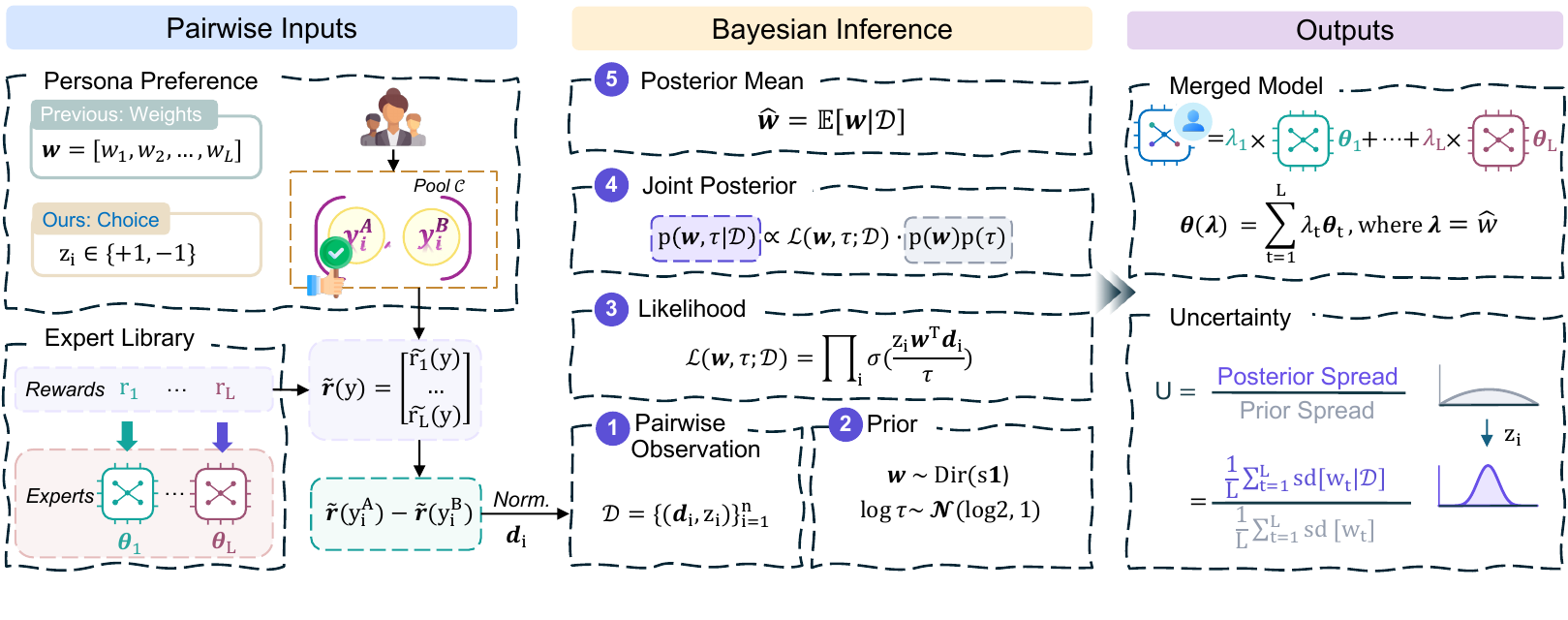}
\end{center}
\vspace{-1.8em}
\caption{\textbf{Bayesian Pairwise Merging.} A persona's pairwise feedback updates a shared prior over reward weights into a posterior, whose mean supplies the merge coefficients and whose spread quantifies their uncertainty. Norm.\ denotes normalization by the infinity norm.}
\vspace{-1.5em}
\label{fig:method}
\end{figure}

Bayesian pairwise merging (BPM) personalizes a model for each \textit{persona}, which represents an individual or a group of users with shared preferences.
From pairwise feedback, \bpm constructs a personalized model and characterizes uncertainty in the inferred reward weights.
\Cref{fig:method} and~\Cref{alg:bpm} summarize the procedure, with key notations listed in~\Cref{tab:app-notation}.

\subsection{Problem Formulation}\label{sec:method-problem}
\paragraph{Expert library and merge family.}
For an input $x$, a model with parameters $\boldsymbol{\theta}$ generates an output $y$, evaluated by $L$ reward functions $r_1(x,y),\dots,r_L(x,y)$, with higher scores indicating better performance on each reward dimension.
The experts $\{\boldsymbol{\theta}_t\}_{t=1}^{L}$, each fine-tuned on reward $r_t$ from a shared base model $\boldsymbol{\theta}_0$, form the shared expert library.
We construct personalized models by combining these expert parameters with the merge coefficients $\boldsymbol{\lambda}$, defining the merge family
\begin{equation}\label{eq:family}
\boldsymbol{\theta}(\boldsymbol{\lambda})
=\sum_{t=1}^{L}\lambda_t\,\boldsymbol{\theta}_t,
\qquad
\boldsymbol{\lambda}\in\Delta^{L-1}
=\Big\{\boldsymbol{\lambda}\in\mathbb R^L:
\lambda_t\ge0,\ \textstyle\sum_{t=1}^{L}\lambda_t=1\Big\},
\end{equation}
where $\Delta^{L-1}$ is the probability simplex.
For our LoRA experts~\citep{hu2022lora}, $\boldsymbol{\theta}_t$ collects the adapter parameters on the shared base model, so \Cref{eq:family} interpolates the adapters (Appendix~\ref{app:merge}).
To personalize a model within this family, \bpm infers a persona's reward weights $\mathbf w\in\Delta^{L-1}$ from pairwise feedback and uses their posterior mean as merge coefficients $\boldsymbol{\lambda}$, following the direct-weighting strategy of Rewarded Soups~\citep[\S2.2.3]{rame2023rewardedsoups}.

\paragraph{Pairwise feedback.}
A persona provides $n$ pairwise comparisons $\{(x_i,y_i^A,y_i^B,z_i)\}_{i=1}^n$, where $y_i^A$ and $y_i^B$ are candidate outputs for the same input $x_i$, and $z_i=+1$ if the persona prefers $y_i^A$ while $z_i=-1$ for $y_i^B$.
\bpm observes only these choices and the reward scores of the compared outputs, without access to the persona description or explicit reward weights.

Because rewards differ in scale~\citep{liu2026gdpo}, we standardize each reward as $\tilde r_t(x,y)=(r_t(x,y)-m_t)/\nu_t$, where the mean $m_t$ and standard deviation $\nu_t$ are estimated on a separate data split.
Writing the standardized reward vector as $\tilde{\mathbf r}(x,y) =[\tilde r_1(x,y),\dots,\tilde r_L(x,y)]^\top$, we represent each comparison by its normalized reward contrast
\begin{equation}\label{eq:contrast}
\mathbf{d}_i=
\frac{\tilde{\mathbf{r}}(x_i,y_i^A)-\tilde{\mathbf{r}}(x_i,y_i^B)}
{\|\tilde{\mathbf{r}}(x_i,y_i^A)-\tilde{\mathbf{r}}(x_i,y_i^B)\|_\infty}
\in[-1,1]^L,
\end{equation}
with $\mathbf d_i=\mathbf 0$ when the two reward vectors are identical.
This normalization preserves the direction of nonzero reward differences while discarding their magnitude, and ensures $|\mathbf w^\top\mathbf d_i|\le1$ for $\mathbf w\in\Delta^{L-1}$.
The resulting observations $\mathcal D=\{(\mathbf d_i,z_i)\}_{i=1}^{n}$ then serve as the input to Bayesian inference.

\subsection{Bayesian Reward-Weight Inference}\label{sec:method-model}
We model pairwise feedback using reward weights $\mathbf w\in\Delta^{L-1}$ and a temperature $\tau>0$ that controls how strongly reward differences influence the probability of preferring one output over another.
A Bradley--Terry likelihood updates the shared prior to obtain a persona-specific posterior.

\paragraph{Bradley--Terry likelihood.}
Following common practice in pairwise preference modeling~\citep{bradley1952rank, shenfeld2025pref}, we model the persona's preference for $y_i^A$ over $y_i^B$ as a Bernoulli outcome with log-odds $\mathbf w^\top\mathbf d_i/\tau$.
The parameter-recovery interpretation below additionally requires correct specification, as formalized in the following assumption.

\begin{assumption}[Bradley--Terry preference model]\label{ass:bt}
There exist reward weights $\mathbf w^\star\in\Delta^{L-1}$ and a temperature $\tau^\star>0$ such that, conditional on the reward contrasts, the preference labels are independent, with
\[
\Pr(z_i=+1\mid\mathbf d_i)
=\sigma\!\left(
\frac{\mathbf w^{\star\top}\mathbf d_i}{\tau^\star}
\right),
\]
where $\sigma(a)=1/(1+\exp(-a))$ is the logistic sigmoid.
\end{assumption}

Under this model, the likelihood of the observed comparisons is
\begin{equation}\label{eq:likelihood}
\mathcal L(\mathbf w,\tau;\mathcal D)
=\prod_{i=1}^{n}
\sigma\!\left(\frac{z_i\,\mathbf w^\top\mathbf d_i}{\tau}\right).
\end{equation}
Because $|\mathbf w^\top\mathbf d_i|\le1$, the log-odds would be restricted to $[-1,1]$ without temperature scaling.
Smaller values of $\tau$ yield sharper choice probabilities, whereas larger values move them toward $1/2$.
The simplex constraint fixes the scale of $\mathbf w$, removing its scaling ambiguity with $\tau$ (Appendix~\ref{app:temperature-identification}).

\paragraph{Prior and posterior.}
All personas within a task share the independent priors $\mathbf w\sim\operatorname{Dir}(s\mathbf 1)$ and $\log\tau\sim\mathcal N(\log 2,1)$, where $\mathbf 1$ is the all-ones vector and $s>0$ controls the prior dispersion.
The symmetric Dirichlet prior assigns equal expected weight to each reward dimension, with $\mathbb E[\mathbf w]=\mathbf 1/L$.
We fix $s$ per task by a variance-matching rule, before any feedback is observed (Appendix~\ref{app:prior}).
Bayes' rule yields the joint posterior
\begin{equation}\label{eq:posterior}
p(\mathbf w,\tau\mid\mathcal D)
\propto
\mathcal L(\mathbf w,\tau;\mathcal D)\,p(\mathbf w)\,p(\tau).
\end{equation}
We sample from this posterior using the No-U-Turn Sampler (NUTS)~\citep{hoffman14nuts} and marginalize over $\tau$ to obtain $p(\mathbf w\mid\mathcal D)$.
Under the conditions of Proposition~\ref{prop:target}, the joint posterior concentrates on the minimizer of expected log-loss, which is $(\mathbf w^\star,\tau^\star)$ when Assumption~\ref{ass:bt} holds.

\subsection{Posterior-Guided Model Merging}\label{sec:method-delivery}
\paragraph{Merged model.}
We use the posterior mean of the reward weights as merge coefficients to construct the personalized model,
\begin{equation}\label{eq:delivery}
\widehat{\mathbf w}=\mathbb E[\mathbf w\mid\mathcal D],
\qquad
\boldsymbol{\lambda}=\widehat{\mathbf w},
\qquad
\boldsymbol{\theta}(\boldsymbol{\lambda})
=\sum_{t=1}^{L}\lambda_t\,\boldsymbol{\theta}_t.
\end{equation}
The posterior mean lies in $\Delta^{L-1}$ and minimizes the posterior expected squared error in estimating $\mathbf w$ (Appendix~\ref{app:posterior-summaries}).
Without feedback, $\widehat{\mathbf w}=\mathbf 1/L$, and \bpm returns the uniform merge.
Personalization therefore requires only posterior inference and one merge, without per-persona policy training.

The posterior mean minimizes squared error in reward-weight estimation, while optimality of the resulting merge additionally requires the following correspondence assumption.
Let $\pi_{\boldsymbol{\theta}}(y\mid x)$ denote the output distribution of a model with parameters $\boldsymbol{\theta}$.

\begin{assumption}[Preference--merge correspondence]\label{ass:coord}
For $\mathbf w^\star$ in Assumption~\ref{ass:bt},
\[
\mathbf w^\star\in
\operatorname*{arg\,max}_{\boldsymbol{\lambda}\in\Delta^{L-1}}
\mathbb E_x
\mathbb E_{y\sim\pi_{\boldsymbol{\theta}(\boldsymbol{\lambda})}
(\cdot\mid x)}
\left[\mathbf w^{\star\top}\tilde{\mathbf r}(x,y)\right],
\]
where $x$ follows the task's input distribution.
\end{assumption}

Under Assumption~\ref{ass:bt} and the conditions of Proposition~\ref{prop:target}, $\widehat{\mathbf w}$ converges to $\mathbf w^\star$ as $n$ grows.
With Assumption~\ref{ass:coord}, the limiting model $\boldsymbol\theta(\mathbf w^\star)$ maximizes the persona's expected weighted reward within the merge family (Appendix~\ref{app:correspondence}).
The correspondence assumption supports this optimality interpretation but is not required to construct the personalized model.

\paragraph{Posterior uncertainty.}
Pairwise feedback constrains reward weights through the observed reward contrasts, with no direct likelihood information in directions orthogonal to their span
(Proposition~\ref{prop:span}, Appendix~\ref{app:identifiability}).
The number of comparisons alone therefore does not determine posterior uncertainty, which we summarize by the ratio of average posterior to prior standard deviations,
\begin{equation}\label{eq:width}
U=
\frac{\tfrac1L\sum_{t=1}^{L}\operatorname{sd}[w_t\mid\mathcal D]}
{\tfrac1L\sum_{t=1}^{L}\operatorname{sd}[w_t]}.
\end{equation}
With $U=1$ as the no-feedback baseline, smaller values indicate lower dispersion under the fitted preference model, whose relationship to personalization performance we examine in~\Cref{sec:analysis}.



\section{Experiment Setup}\label{sec:setup}

\subsection{Tasks and Model Libraries}\label{sec:setup-tasks}
\paragraph{Tasks and datasets.}
We evaluate \bpm on three tasks spanning text-only and vision-language settings. 
For \textsc{Summarization}, models generate a short impression from the comprehensive radiology report findings in the MIMIC-CXR dataset~\citep{PhysioNet-mimic-cxr}.
For \textsc{Image Captioning}, models use free-form text to describe images from COCO~\citep{lin2014coco}. 
For \textsc{Story Generation}, models write short stories in response to instructions from Writing Prompts~\citep{fan2018hierarchical}.

\paragraph{Expert libraries.}
For each task, we use standard evaluation metrics as reward dimensions, scored by computed metrics or an LLM scorer (gpt-oss-20b) separate from the persona judges, and train one LoRA expert~\citep{hu2022lora} per dimension from a shared base model using GRPO~\citep{deepseek-math}.
The libraries comprise $29$ reward-specialized experts based on Qwen3.5-2B~\citep{qwen35} for radiology report summarization, $21$ on Qwen3-VL-4B~\citep{bai2025qwen3vl} for image captioning, and $16$ on Ministral-3-3B~\citep{liu2026ministral3} for story generation.
To assess \bpm's robustness to variations in training, we construct additional \textsc{Summarization} libraries using GRPO with two more seeds, RLOO~\citep{ahmadian2024rloo}, or DPO~\citep{rafailov2023direct}.
All libraries are shared across personas, with persona-specific feedback determining the merge coefficients.
Reward definitions and training details are provided in Appendices~\ref{app:rewards} and~\ref{app:training}, respectively.

\subsection{Personalization and Evaluation Protocol}\label{sec:setup-protocol}
\paragraph{Persona simulation.}
We simulate personas through either natural-language descriptions or explicit weights over rubric criteria.
For \textit{prompt-based personas}, we use Qwen3.5-27B~\citep{qwen35} to simulate six personas for \textsc{Summarization} and seven for \textsc{Story Generation}, while for \textsc{Image Captioning}, six personas are simulated with Qwen3-VL-32B~\citep{qwen3technicalreport}.
Both models are conditioned on system prompts describing diverse preferences, including those not explicitly represented by the reward library (Appendix~\ref{app:personas}).
For \textit{rubric-based personas}, we construct five for \textsc{Story Generation} using AutoRubric~\citep{rao2026autorubric} for weighted rubric evaluation, with adjustable criterion weights specifying preferences.
These weights remain hidden from \bpm and can be varied to study how personalized models follow changes in preference (Appendix~\ref{app:forma}).

\paragraph{Preference assignment.}
\textit{Prompt-based personas} use the same LLM and persona description
to provide pairwise feedback and evaluate model outputs on
disjoint test inputs.
Each response pair is judged in both presentation orders;
feedback collection retains only valid, non-tied choices
that agree across orders.
\textit{Rubric-based personas} assign preferences according to
persona-weighted rubric scores, with tied or invalid comparisons
excluded from feedback.
The feedback budget $n$ counts retained comparisons rather than
individual judge queries.

\paragraph{Pool $\mathcal{C}$ construction and feedback simulation.}\label{sec:setup-pool}
For each task, the experts, uniform merge, and $20$ random merges generate responses on $250$ inputs disjoint from reward standardization and testing.
We sample four pairs per input and obtain persona judgments using the protocol above.
Three feedback repetitions are sampled independently and nested across budgets within each repetition, keeping the expert library fixed (Appendices~\ref{app:data}, \ref{app:pool-construction}, and~\ref{app:seeds})

\paragraph{Compared methods.}
We compare \bpm with methods that personalize through expert weighting (\textsc{rm} and \textsc{l2w}), response selection (\textsc{bt-bon}), prompting (\textsc{icai}), or policy training (\textsc{dpo}), alongside the \textsf{base} policy and uniform merge (\textsc{rs}).
For each persona and budget, methods that use pairwise feedback receive the same sampled subsets as \bpm.
\begin{itemize}[leftmargin=*,itemsep=0pt,topsep=0pt,partopsep=0pt]
    \item \textit{\textbf{Uniform merge (\textsc{rs})}} uses the equal-weight configuration of Rewarded Soups ~\citep{rame2023rewardedsoups}, $\boldsymbol{\theta}(\mathbf 1/L)$, which is also \bpm's policy before observing persona-specific pairwise feedback. And 
    
    \item \textit{\textbf{\textsc{base} policy}} is the shared base model $\boldsymbol{\theta}_0$ before reward-specific expert training.

    \item \textit{\textbf{Ridge merge (\textsc{rm})}} fits an unconstrained coefficient vector with $\ell_2$-regularized logistic regression and maps it to the simplex through a softmax, returning a penalized point estimate in place of \bpm's posterior, which represents the point-estimate preference recovery of PReF~\citep{shenfeld2025pref} and LoRe~\citep{bose2025lore}, adapted to our merge family.

    \item \textit{\textbf{Maximum likelihood and maximum a posteriori merging (\textsc{mle}/\textsc{map})}} fit the same Bradley--Terry preference model as \bpm. \textsc{mle} omits the priors, whereas \textsc{map} uses the joint posterior mode in transformed coordinates, including the transformation Jacobian (Appendix~\ref{app:analysis}).
    
    \item \textit{\textbf{Adapted Spectral Souping (\textsc{ss-a})}} uses the coefficient estimator of~\citet{chow2026spectral} with the same offline preference pairs, replacing policy-derived features with standardized reward-score differences.
    It fits signed coefficients under an $\ell_1$ constraint and uses exact delta merging to accommodate negative coefficients (Appendix~\ref{app:further}).
    
    \item \textit{\textbf{Preference-based reranking (\textsc{bt-bon})}} fits a Bradley-Terry reward model to the preference pairs and selects the highest-scoring output among eight samples from the uniform merge~\citep{nakano2021webgpt, gui2024bonbon}, which applies personalization through output selection.

    \item \textit{\textbf{Language-to-Weights (\textsc{l2w})}} follows the language-guided weighting approach of Promptable Behaviors~\citep{Hwang2024promptablebehavior}, using gpt-oss-20b~\citep{openai2025gptoss120bgptoss20bmodel} to infer reward weights from the persona description, which is not provided to \bpm.
    The inferred weights are normalized and used as merge coefficients.

    \item \textit{\textbf{Inverse Constitutional AI (\textsc{icai})}} uses gpt-oss-20b to infer a constitution from the preference pairs, following~\citet{findeis2025icai}.
    The same constitution is supplied as a system prompt to either the base policy \textsc{icai-b}, as in the original setting, or the uniform merge \textsc{icai-m}, \bpm's starting policy.
    
    \item \textit{\textbf{Direct preference optimization (\textsc{dpo})}} fine-tunes the base policy on each persona's preference pairs~\citep{rafailov2023direct}, requiring per-persona training that \bpm avoids by merging shared experts.
    This compared method is distinct from the DPO-trained expert libraries described above.
\end{itemize}
Further details on compared method configurations, training budgets, and hyperparameter selection are provided in Appendix~\ref{app:baselines}.

\paragraph{Test evaluation.}
We evaluate on $451$ radiology reports, $500$ images, and $500$ story prompts with the same two-order judgment protocol.
Decided win rates are $N_{\mathrm{win}}/(N_{\mathrm{win}}+N_{\mathrm{loss}})$; supplementary scores retain valid ties and order-inconsistent judgments at half credit, with assignment sensitivity reported in Appendix~\ref{app:evaluation-accounting}.
Macro rates equally weight personas and average over feedback draws.
The 95\% input-bootstrap intervals condition on the observed personas, feedback subsets, and expert library (Appendix~\ref{app:ci}).
\section{Results}\label{sec:results}

\subsection{Personalization Performance}\label{sec:results-main}
\definecolor{highcolor}{HTML}{137333}
\definecolor{lowcolor}{HTML}{A50E0E}
\definecolor{midcolor}{HTML}{FFFFFF}
\def\mrbcolorscale{50}
\def\mrbmaxintensity{70}
\def\mrbhldiff#1#2{%
  \begingroup
  \pgfmathsetmacro{\diffval}{#1 - #2}%
  \pgfmathtruncatemacro{\mrbint}{round(min(1, abs(\diffval) / \mrbcolorscale) * \mrbmaxintensity)}%
  \ifdim \diffval pt > 0pt
    \edef\mrbcol{highcolor!\mrbint!midcolor}%
    \expandafter\cellcolor\expandafter{\mrbcol}#1%
  \else
    \ifdim \diffval pt < 0pt
      \edef\mrbcol{lowcolor!\mrbint!midcolor}%
      \expandafter\cellcolor\expandafter{\mrbcol}#1%
    \else
      #1%
    \fi
  \fi
  \endgroup
}

\newcommand{\mrbci}[1]{{\footnotesize [#1]}}

\begin{table}[t]
\caption{\textbf{Main results.} \bpm's macro decided win rate (\%) against each method, with 95\% CIs, averaged over three sampled feedback subsets. 
Rows are not comparable with each other.
\textcolor{highcolor}{Green} favors \bpm and \textcolor{lowcolor}{red} favors the compared method, with intensity proportional to the distance from $50$.
}
\vspace{-0.8em}
\centering
\resizebox{\textwidth}{!}{%
\begin{tabular}{lcccccc}
\toprule
 & \multicolumn{2}{c}{\textbf{\textsc{Summarization}}} & \multicolumn{2}{c}{\textbf{\textsc{Image Captioning}}} & \multicolumn{2}{c}{\textbf{\textsc{Story Generation}}} \\
\cmidrule(lr){2-3}\cmidrule(lr){4-5}\cmidrule(lr){6-7}
& $n=100$ & $n=500$ & $n=100$ & $n=500$ & $n=100$ & $n=500$ \\
\midrule
\textsc{RS} & \mrbhldiff{91.7}{50}\,\mrbci{90.3, 93.1} & \mrbhldiff{88.4}{50}\,\mrbci{86.2, 90.4} & \mrbhldiff{77.1}{50}\,\mrbci{74.7, 79.3} & \mrbhldiff{78.7}{50}\,\mrbci{76.6, 80.7} & \mrbhldiff{64.3}{50}\,\mrbci{62.5, 66.1} & \mrbhldiff{76.1}{50}\,\mrbci{74.4, 77.8} \\
\textsc{Base} & \mrbhldiff{76.8}{50}\,\mrbci{74.6, 79.0} & \mrbhldiff{76.7}{50}\,\mrbci{74.0, 79.3} & \mrbhldiff{74.6}{50}\,\mrbci{72.2, 76.7} & \mrbhldiff{76.4}{50}\,\mrbci{74.4, 78.2} & \mrbhldiff{71.3}{50}\,\mrbci{69.4, 73.1} & \mrbhldiff{81.4}{50}\,\mrbci{79.9, 82.8} \\
\textsc{RM} & \mrbhldiff{74.9}{50}\,\mrbci{73.2, 76.7} & \mrbhldiff{84.1}{50}\,\mrbci{81.8, 86.2} & \mrbhldiff{72.1}{50}\,\mrbci{70.3, 74.0} & \mrbhldiff{75.3}{50}\,\mrbci{73.5, 77.0} & \mrbhldiff{57.3}{50}\,\mrbci{56.1, 58.6} & \mrbhldiff{71.5}{50}\,\mrbci{70.2, 72.7} \\
\textsc{MLE}
& \mrbhldiff{28.4}{50}\,\mrbci{26.6, 30.3}
& \mrbhldiff{43.9}{50}\,\mrbci{38.8, 49.0}
& \mrbhldiff{44.7}{50}\,\mrbci{43.0, 46.4}
& \mrbhldiff{45.4}{50}\,\mrbci{43.3, 47.6}
& \mrbhldiff{43.4}{50}\,\mrbci{42.2, 44.7}
& \mrbhldiff{46.3}{50}\,\mrbci{45.0, 47.6} \\

\textsc{MAP}
& \mrbhldiff{28.9}{50}\,\mrbci{26.6, 31.2}
& \mrbhldiff{51.1}{50}\,\mrbci{46.1, 56.0}
& \mrbhldiff{57.9}{50}\,\mrbci{55.5, 60.2}
& \mrbhldiff{58.8}{50}\,\mrbci{56.4, 61.3}
& \mrbhldiff{49.5}{50}\,\mrbci{48.2, 50.7}
& \mrbhldiff{51.9}{50}\,\mrbci{50.8, 53.1} \\

\textsc{SS-A}\tablefootnote{\textsc{ss-a} uses an adapted coefficient estimator
with exact delta merging while \bpm uses factor-averaged LoRA merging. Details and a same-operator control appear in Appendix~\ref{app:ss}.}
& \mrbhldiff{24.3}{50}\,\mrbci{21.7, 26.9}
& \mrbhldiff{50.2}{50}\,\mrbci{44.6, 55.9}
& \mrbhldiff{15.6}{50}\,\mrbci{14.5, 16.8}
& \mrbhldiff{19.2}{50}\,\mrbci{17.7, 20.7}
& \mrbhldiff{33.2}{50}\,\mrbci{32.0, 34.5}
& \mrbhldiff{33.7}{50}\,\mrbci{32.3, 35.0} \\
\textsc{BT-BoN} & \mrbhldiff{82.5}{50}\,\mrbci{80.7, 84.1} & \mrbhldiff{85.1}{50}\,\mrbci{82.7, 87.3} & \mrbhldiff{67.6}{50}\,\mrbci{65.9, 69.4} & \mrbhldiff{69.4}{50}\,\mrbci{67.7, 71.0} & \mrbhldiff{56.8}{50}\,\mrbci{55.2, 58.5} & \mrbhldiff{67.1}{50}\,\mrbci{65.3, 68.7} \\
\midrule
\textsc{L2W} & \mrbhldiff{59.6}{50}\,\mrbci{57.6, 61.6} & \mrbhldiff{81.9}{50}\,\mrbci{78.4, 85.3} & \mrbhldiff{72.0}{50}\,\mrbci{69.7, 74.2} & \mrbhldiff{76.2}{50}\,\mrbci{74.5, 77.9} & \mrbhldiff{58.9}{50}\,\mrbci{57.3, 60.4} & \mrbhldiff{72.4}{50}\,\mrbci{71.0, 73.8} \\
\textsc{ICAI-M} & \mrbhldiff{74.4}{50}\,\mrbci{72.6, 76.4} & \mrbhldiff{88.6}{50}\,\mrbci{86.7, 90.4} & \mrbhldiff{48.2}{50}\,\mrbci{46.6, 49.8} & \mrbhldiff{59.5}{50}\,\mrbci{58.3, 60.7} & \mrbhldiff{44.0}{50}\,\mrbci{42.7, 45.3} & \mrbhldiff{52.6}{50}\,\mrbci{51.3, 53.9} \\
\textsc{ICAI-B}  & \mrbhldiff{63.9}{50}\,\mrbci{62.0, 65.9} & \mrbhldiff{73.2}{50}\,\mrbci{70.9, 75.4} & \mrbhldiff{50.2}{50}\,\mrbci{48.8, 51.6} & \mrbhldiff{60.3}{50}\,\mrbci{59.2, 61.4} & \mrbhldiff{52.0}{50}\,\mrbci{50.6, 53.4} & \mrbhldiff{61.0}{50}\,\mrbci{59.6, 62.4} \\
\textsc{DPO} & \mrbhldiff{37.9}{50}\,\mrbci{35.6, 40.0} & \mrbhldiff{57.6}{50}\,\mrbci{54.2, 60.6} & \mrbhldiff{7.9}{50}\,\mrbci{7.1, 8.8} & \mrbhldiff{16.0}{50}\,\mrbci{14.7, 17.3} & \mrbhldiff{36.1}{50}\,\mrbci{34.2, 38.0} & \mrbhldiff{49.0}{50}\,\mrbci{46.9, 51.3} \\
\bottomrule
\end{tabular}%
}
\vspace{-1em}
\label{tab:main}
\end{table}

\begin{table}[t]
\centering
\def\ci#1#2{%
  \begingroup
  \pgfmathsetmacro{\diffval}{#1 - 50}%
  \pgfmathtruncatemacro{\mrbint}{%
    round(min(1, abs(\diffval) / \mrbcolorscale)
          * \mrbmaxintensity)}%
  \ifdim \diffval pt > 0pt
    \edef\mrbcol{highcolor!\mrbint!midcolor}%
  \else
    \ifdim \diffval pt < 0pt
      \edef\mrbcol{lowcolor!\mrbint!midcolor}%
    \else
      \def\mrbcol{midcolor}%
    \fi
  \fi
  \expandafter\cellcolor\expandafter{\mrbcol}%
  \shortstack{#1\\{\footnotesize [#2]}}%
  \endgroup
}

\caption{\textbf{Robustness results.} 
BPM's macro-weighted decided win rate (\%) with 95\% CI. Cell colors as in Table~\ref{tab:main}.
\textbf{(a)} Experts trained with GRPO, RLOO, or DPO on the summarization task. 
Each library is compared with its own uniform merge, so columns do not rank training algorithms.
\textbf{(b)} Rubric-based personas replace the prompted personas, on \textsc{Story generation}.
}
\vspace{-0.8em}
\resizebox{\textwidth}{!}{%
\begin{tabular}{@{}cccccccc@{}}
\toprule
 & \multicolumn{3}{c}{\textbf{(a) \textsc{Expert Training Algorithm}}}
 & \multicolumn{4}{c}{\textbf{(b) \textsc{Rubric-based Persona}}}\\
\cmidrule(lr){2-4}\cmidrule(lr){5-8}
 & GRPO & RLOO & DPO
 & \textsc{Uniform Merge} & \textsc{Base Policy} & \textsc{Ridge-Merge} & \textsc{BT-BoN} \\
\midrule
$n=100$
 & \ci{91.7}{90.3,93.1}
 & \ci{93.9}{92.6,95.1}
 & \ci{77.7}{74.2,81.1}
 & \ci{65.7}{63.7,67.7}
 & \ci{70.8}{69.0,72.8}
 & \ci{60.4}{59.1,61.7}
 & \ci{63.4}{61.7,65.1} \\
$n=500$
 & \ci{88.4}{86.2,90.4}
 & \ci{97.7}{96.9,98.4}
 & \ci{90.3}{88.2,92.1}
 & \ci{74.7}{72.8,76.5}
 & \ci{77.8}{76.2,79.3}
 & \ci{71.4}{70.2,72.7}
 & \ci{71.1}{69.3,72.7} \\
\bottomrule
\end{tabular}%
}
\vspace{-1.5em}
\label{tab:ablation}
\end{table}

\paragraph{Effectiveness of feedback-derived coefficients.}
\bpm improves on uniform merging across all three tasks at both feedback budgets (\Cref{tab:main}), showing that pairwise feedback supplies useful coefficients for personalizing the shared expert library.
On \textsc{Story Generation}, gains over uniform merging also persist with \texttt{gpt-oss-20} as an alternative judge (\Cref{fig:analysis}e).

\paragraph{Comparison with alternative methods.}
\bpm outperforms ridge merge and preference-based reranking, whereas \textsc{mle} and adapted \textsc{ss} often yield more preferred outputs, and comparisons with \textsc{map} vary by task and budget.
The lower block compares methods that use an auxiliary LLM for personalization or require per-persona policy training,
\S\ref{sec:analysis} \S\ref{sec:analysis} examines what \bpm's posterior additionally provides in terms of coefficient stability and uncertainty.

\paragraph{Limited feedback supports personalization.}
\Cref{fig:analysis}c shows gains over uniform merge with as few as $10$ comparisons in all three tasks, demonstrating that even a small number of choices can supply useful coefficients.
These gains generally increase with the feedback budget, supporting personalization from limited feedback and further improvement as more choices become available.

%
\begin{figure}[t]
\begin{center}
\includegraphics[width=\textwidth]{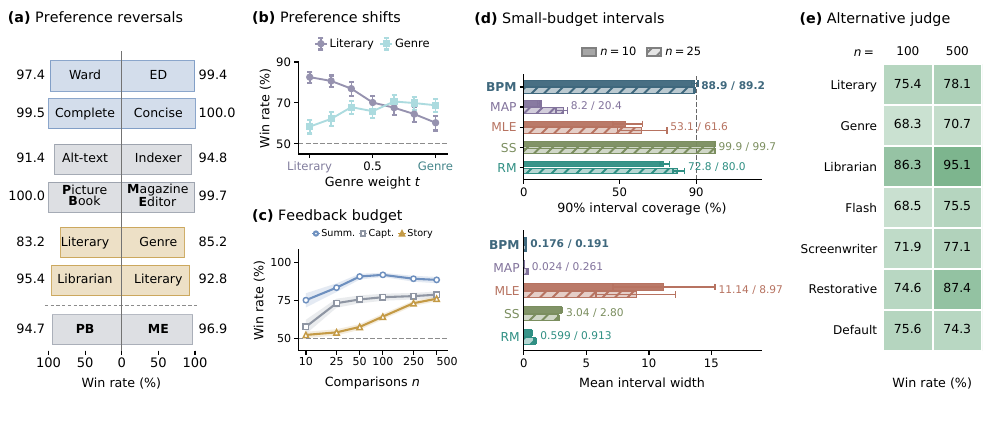}
\end{center}
\vspace{-1.8em}
\caption{\textbf{Preference behavior and coefficient uncertainty.}
\textbf{(a)}~Each persona's win rate for its corresponding model ($n=500$).
The bottom row uses human judgments ($n=100)$.
\textbf{(b)}~Win rates against uniform merge as feedback preferences shift ($n=500$).
\textbf{(c)}~Macro win rates across budgets.
\textbf{(d)}~Coverage (top) and mean width (bottom) of nominal $90\%$ intervals across tasks, whiskers spanning task ranges.
\textbf{(e)}~Win rates against uniform merge using an alternative LLM for judgements.}
\label{fig:analysis}
\end{figure}

\subsection{Preference Following}\label{sec:results-follow}
\paragraph{\bpm reflects persona-specific preferences.}
To assess whether inferred weights translate feedback into preference-specific model behavior, we select two simulated persona pairs per task and have both personas judge the same outputs from their respective models.
Each persona favors its own model in all six pairs (\Cref{fig:analysis}a).
A single author-annotator also exhibits preference reversals under the \textit{Picture Book (PB)} and \textit{Magazine Editor (ME} descriptions in a human proof-of-concept (Appendix~\ref{app:human}).
Length- and readability-matched analyses appear in Appendix~\ref{app:length}.

\paragraph{\bpm follows controlled preference shifts.}
Rubric-based personas allow us to vary feedback preferences while keeping the evaluation criteria fixed.
As shown in~\Cref{tab:ablation}b, \bpm retains its macro advantage over the compared methods on the task of story generation.
As feedback shifts from \textit{Literary} toward \textit{Genre}, win rates against uniform decrease under the fixed \textit{Literary} judge and increase overall under the fixed \textit{Genre} judge (\Cref{fig:analysis}b), supporting the responsiveness of the merged models to changes in pairwise feedback preferences.
\section{Analysis}\label{sec:analysis}

\subsection{Coefficient Stability and Output Preference}
\label{sec:analysis-stabiliy}

\paragraph{Posterior means reduce coefficient variation under limited feedback.}
To assess sensitivity to feedback sampling, we measure the mean pairwise $\ell_1$ distance between coefficients fitted to three feedback subsets, averaged over personas.
\Cref{fig:uncertainty}a shows lower variation for \bpm than \textsc{mle} at $n\leq100$ in all three tasks.
Compared with adapted \textsc{ss}, the advantage holds across tasks at $n=10$, with task-dependent differences at larger budgets.
Appendix~\ref{app:point-estimates} reports paired intervals.


\paragraph{Coefficient stability and output preference are distinct.}
\bpm improves on uniform merging while reducing coefficient
variation relative to \textsc{mle} under limited feedback.
This stability advantage does not imply higher output preference:
\textsc{mle} and adapted \textsc{ss} often yield more preferred
outputs.
Conversely, \textsc{rm} can be more stable than \bpm but is less preferred at both main-table budgets.
\bpm thus provides useful personalization with greater stability under limited feedback and a posterior over plausible
weights, rather than a uniformly superior stability--preference
trade-off (Appendix~\ref{app:point-estimates}).

\subsection{Uncertainty under Limited Feedback}\label{sec:analysis-uncertainty}

\paragraph{\bpm quantifies support for alternative reward trade-offs.}
Starting from a shared prior, the \textit{Literary} and \textit{Librarian} posteriors favor stylistically rich expression and clear, accessible storytelling, respectively (Appendix~\ref{app:uncertainty-performance}).
Their means supply persona-specific merge coefficients, while the full posteriors quantify the relative support for alternative reward trade-offs under the model and prior.
\bpm thus provides both a concrete personalization choice and a probabilistic account of which alternatives remain plausible given the feedback.

\paragraph{\bpm provides compact intervals with near-nominal coverage in matched simulations.}
In simulations under \bpm's model and prior, nominal 90\% intervals for $\boldsymbol\beta=\mathbf w/\tau$ achieve coverage of 88.9\% and 89.2\% after 10 and 25 comparisons (\Cref{fig:analysis}d).
\textsc{map}, \textsc{mle}, and \textsc{rm} bootstrap intervals under-cover, whereas \textsc{rm} Laplace and \textsc{ss} Gaussian intervals are substantially wider and approach full coverage.
Prior-only intervals have comparable coverage and widths, and prior perturbations reduce coverage in some settings (Appendix~\ref{app:coverage}).
\bpm thus provides compact uncertainty summaries with near-nominal coverage under the assumed model and prior, with limited evidence of interval contraction at these budgets.

\paragraph{Posterior uncertainty contextualizes personalization performance.}
To relate uncertainty to model behavior, \Cref{fig:uncertainty}b compares $U$ with win rates across personas within each task and budget.
Lower uncertainty is associated with higher win rates, while the association is weak on \textsc{Story Generation}.
\bpm thus complements preference scores by quantifying the coefficient ambiguity remaining for each persona (Appendix~\ref{app:uncertainty-performance}).

\begin{figure}[t]
\begin{center}
\includegraphics[width=\textwidth]{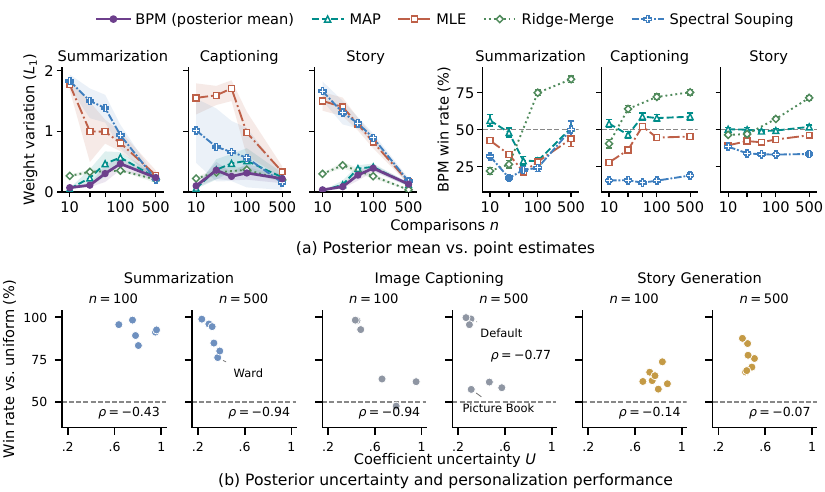}
\end{center}
\vspace{-1em}
\caption{\textbf{Posterior analysis.}
(a) Left: merge-coefficient variation across feedback subsets. Right: \bpm's decided win rates against alternative estimators (Appendix~\ref{app:point-estimates}).
(b) Posterior uncertainty $U$ versus decided win rate against uniform merging.
Each point is one persona, with both quantities averaged over three feedback subsets; $\rho$ is the within-task, fixed-budget Spearman correlation.
}
\vspace{-1.5em}
\label{fig:uncertainty}
\end{figure}

\subsection{From Inferred Weights to Model Behavior}

\paragraph{Posteriors reveal persona-specific reward trade-offs.}
To examine how feedback updates the shared prior, \Cref{fig:posterior-story} shows six reward-weight marginals for two personas on \textsc{Story Generation} at $n=500$.
The \textit{Literary} posterior favors expressive writing, whereas the
\textit{Librarian} posterior favors clear, coherent storytelling.
Beyond supplying merge coefficients through posterior means, \bpm quantifies support for alternative weight values under the model and prior.

\paragraph{Gains persist across expert libraries.}
As shown in~\Cref{tab:ablation}a, \bpm outperforms each library's own uniform merge with both RLOO- and DPO-trained experts on \textsc{Summarization}.
The gains hold for every persona at both feedback budgets, supporting feedback-based personalization of expert libraries trained with different optimization methods.

\paragraph{Inferred weights support policy training.}\label{sec:analysis-retrain}
Using inferred \textit{Ward} and \textit{Emergency Doctor (ED)} weights yields \textsc{Summarization} policies preferred over uniform-weight training and training with the other persona's weights (Appendix~\ref{app:retrain}), extending the usefulness of \bpm's inferred weights beyond model merging.
\section{Conclusion}
We present \bpm, which obtains the missing coefficients for model personalization by inferring a posterior over reward weights from pairwise choices and using its mean to merge shared experts, without per-user policy training.
Across three generation tasks, the resulting models improve on uniform merging and reflect contrasting simulated persona preferences, while controlled feedback shifts produce corresponding behavioral changes.
Gains across expert libraries trained with different optimization methods further support the usefulness of feedback-derived coefficients.
Beyond model construction, \bpm provides coefficients that vary less than \textsc{mle} across limited-feedback subsets and a posterior quantifying plausible reward weights.
In simulations under its model and prior, compact intervals for temperature-scaled reward weights achieve near-nominal coverage.
\bpm thus connects pairwise feedback to personalized models while making the remaining ambiguity in the inferred preferences explicit.

\subsection*{AI use statement}
In this work, we used generative AI tools to assist with experimental design, method implementation, and mathematical derivations.
LLMs were also used to simulate personas to provide pairwise preference judgments and evaluate generated outputs, as documented in the experimental setup and appendix.
Additionally, we used generative AI tools to assist with literature discovery, manuscript organization, drafting, and language editing.
All AI-assisted work was reviewed by the authors.
AI-generated code was checked and tested by at least one author, and the manuscript was proofread by all authors.
We take responsibility for the final content of this work, including its text, mathematical claims, experimental results, and AI-assisted artifacts.

\subsection*{Ethics statement}
This research involves human annotations to construct pairwise feedback. 
All annotation tasks were conducted by the authors of this paper, who participated voluntarily and with full knowledge of the study’s purpose, procedures, and intended use of the data. 
No external crowdsourcing or paid annotation platforms were employed. 
The study does not involve sensitive personal data, human subjects outside of the annotation task, or applications that raise privacy, security, or legal concerns. We also follow the standard research ethics protocols of our institution, with explicit approval from the IRB, for all internal annotation efforts.
The research complies with the ICLR Code of Ethics, and no conflicts of interest or sponsorship concerns are associated with this work.

\subsection*{Reproducibility statement}
The main paper and appendix describe the components needed to reproduce \bpm, including reward normalization, prior specification, posterior inference, and model merging, together with datasets and splits, expert training configurations, persona prompts, feedback collection procedures, and evaluation protocols.
We also document the compared methods' configuration, simulation settings, and confidence interval estimation procedures.
The supplementary code provides a synthetic inference demo and core pipeline components.
We will release the full implementation and experiment configurations to support reproducibility and future research.
These efforts are intended to enable the community to replicate our experiments and build upon our findings.




\bibliography{references}
\bibliographystyle{iclr2027_conference}

\appendix
\section{Method Details and Derivations}\label{app:derivations}
\subsection{Notation}\label{app:notation}
\begin{table}[t]
\centering
\caption{\textbf{Notation used in Bayesian Pairwise Merging.} Reward weights describe preference trade-offs, whereas merge coefficients determine how experts are combined.}
\label{tab:app-notation}
\small
\renewcommand{\arraystretch}{1.08}
\begin{tabular}{@{}p{.28\linewidth}p{\dimexpr.72\linewidth-2\tabcolsep\relax}@{}}
\toprule
\textbf{Symbol} & \textbf{Definition} \\
\midrule
$x, y$ & Input and generated output. \\
$L, t$ & Number of rewards and reward/expert index. \\
$\boldsymbol{\theta},\ \boldsymbol{\theta}_0$ & Model and shared base-model parameters. \\
$\{\boldsymbol{\theta}_t\}_{t=1}^{L}$ & Shared library of reward-specialized experts. \\
$\Delta^{L-1}$ & Probability simplex in $\mathbb R^L$. \\
$\boldsymbol{\lambda}\in\Delta^{L-1}$ & Merge coefficients, set to $\widehat{\mathbf w}$ in BPM. \\
$\boldsymbol{\theta}(\boldsymbol{\lambda})$ & Merged-model parameters, $\sum_{t=1}^{L}\lambda_t\boldsymbol{\theta}_t$. \\
$\pi_{\boldsymbol{\theta}}(y\mid x)$ & Model output distribution given $x$. \\
\midrule
$r_t(x,y)$ & Reward score for dimension $t$. \\
$m_t,\ \nu_t$ & Mean and standard deviation used to standardize reward $t$. \\
$\tilde r_t(x,y),\ \tilde{\mathbf r}(x,y)$ & Standardized reward score and reward vector. \\
$n, i$ & Number of comparisons and comparison index. \\
$y_i^A,\ y_i^B$ & Candidate outputs compared for input $x_i$. \\
$z_i\in\{-1,+1\}$ & Preference label, $+1$ for $A$ and $-1$ for $B$. \\
$\mathbf d_i\in[-1,1]^L$ & Normalized reward contrast for comparison $i$. \\
$\mathcal D$ & Inference data $\{(\mathbf d_i,z_i)\}_{i=1}^{n}$. \\
\midrule
$\mathbf w\in\Delta^{L-1}$ & Persona-specific reward weights. \\
$\tau>0$ & Bradley--Terry choice temperature. \\
$\mathbf w^\star,\ \tau^\star$ & Generating reward weights and temperature under Assumption~\ref{ass:bt}. \\
$\sigma(a)$ & Logistic sigmoid. \\
$\mathbf 1$ & All-ones vector in $\mathbb R^L$. \\
$s>0$ & Dirichlet parameter controlling prior dispersion. \\
$\mathcal L(\mathbf w,\tau;\mathcal D)$ & Bradley--Terry likelihood of $\mathcal D$. \\
$p(\mathbf w,\tau\mid\mathcal D)$ & Joint posterior over $\mathbf w$ and $\tau$. \\
$p(\mathbf w\mid\mathcal D)$ & Marginal posterior over $\mathbf w$. \\
$\widehat{\mathbf w}$ & Posterior mean $\mathbb E[\mathbf w\mid\mathcal D]$. \\
$U$ & Ratio of average posterior to prior standard deviations of $\mathbf w$. \\
\bottomrule
\end{tabular}
\end{table}

\begin{algorithm}[t]
\caption{Bayesian Pairwise Merging (BPM) for a persona}
\label{alg:bpm}
\small
\begin{algorithmic}[1]
\REQUIRE experts $\{\boldsymbol{\theta}_t\}_{t=1}^{L}$;
feedback $\{(x_i,y_i^A,y_i^B,z_i)\}_{i=1}^{n}$;
standardized rewards $\tilde{\mathbf r}$; parameter $s$

\STATE $\mathcal D\gets\{(\mathbf d_i,z_i)\}_{i=1}^{n}$
\hfill$\triangleright$~\Cref{eq:contrast}
\STATE sample $p(\mathbf w,\tau\mid\mathcal D)$ using NUTS;
check convergence
\hfill$\triangleright$~\Cref{eq:posterior}
\STATE compute $\widehat{\mathbf w}
=\mathbb E[\mathbf w\mid\mathcal D]$ and $U$
\hfill$\triangleright$~\Cref{eq:delivery,eq:width}
\STATE $\boldsymbol{\lambda}\gets\widehat{\mathbf w}$,
$\quad\boldsymbol{\theta}(\boldsymbol{\lambda})
\gets\sum_{t=1}^{L}\lambda_t\,\boldsymbol{\theta}_t$
\hfill$\triangleright$~\Cref{eq:family}
\RETURN $\boldsymbol{\theta}(\boldsymbol{\lambda})$ and $U$
\end{algorithmic}
\end{algorithm}

\subsection{Merging LoRA Experts}
\label{app:merge}

Each expert consists of a LoRA adapter on a shared frozen base model~\citep{hu2022lora}.
For an adapted layer, expert $t$ has the effective weight matrix
\[
\mathbf W_t=\mathbf W_0+\gamma\mathbf B_t\mathbf A_t,
\]
where $\mathbf W_0$ is the base weight matrix, $\mathbf A_t$ and $\mathbf B_t$ are the adapter factors, and $\gamma$ is the shared LoRA scaling factor.
The experts use matching adapter shapes and corresponding layers.

In \Cref{eq:family}, $\boldsymbol{\theta}_t$ represents the model parameters in the shared-base-and-adapter parameterization.
Since $\sum_{t=1}^{L}\lambda_t=1$, averaging these parameters preserves the shared base and combines the corresponding adapter factors,
\begin{equation}\label{eq:app-factor-merge}
\begin{gathered}
\overline{\mathbf A}=\sum_{t=1}^{L}\lambda_t\mathbf A_t,
\qquad
\overline{\mathbf B}=\sum_{t=1}^{L}\lambda_t\mathbf B_t,\\
\mathbf W_{\mathrm{factor}}(\boldsymbol{\lambda})
=\mathbf W_0+\gamma\overline{\mathbf B}\,\overline{\mathbf A}.
\end{gathered}
\end{equation}
Unless otherwise stated, merges use factor-averaged LoRA adapters.
Adapted SS and the exact-delta controls in Appendix~\ref{app:further}
use Equation~\ref{eq:app-delta-merge}.

\paragraph{Relation to averaging effective weight matrices.}
Averaging the effective matrices instead would give
\begin{equation}\label{eq:app-delta-merge}
\mathbf W_{\mathrm{delta}}(\boldsymbol{\lambda})
=\sum_{t=1}^{L}\lambda_t\mathbf W_t
=\mathbf W_0+\gamma\sum_{t=1}^{L}\lambda_t\mathbf B_t\mathbf A_t.
\end{equation}
The difference between the two constructions is
\begin{equation}\label{eq:app-merge-difference}
\mathbf W_{\mathrm{factor}}(\boldsymbol{\lambda})
-\mathbf W_{\mathrm{delta}}(\boldsymbol{\lambda})
=-\gamma\sum_{t=1}^{L}\lambda_t
(\mathbf B_t-\overline{\mathbf B})
(\mathbf A_t-\overline{\mathbf A}).
\end{equation}
Expanding the right-hand side gives
$\gamma(\overline{\mathbf B}\,\overline{\mathbf A}-\sum_t\lambda_t\mathbf B_t\mathbf A_t)$.
The constructions therefore agree when all $\mathbf A_t$ are identical or all $\mathbf B_t$ are identical, but need not agree after separate expert training.
Both are continuous in $\boldsymbol{\lambda}$.

\subsection{Reward Weights and Choice Temperature}
\label{app:temperature-identification}

The Bradley--Terry likelihood depends on $\mathbf w$ and $\tau$ through $\boldsymbol{\beta}=\mathbf w/\tau$.
Without a scale constraint on $\mathbf w$, the pairs $(\mathbf w,\tau)$ and $(c\mathbf w,c\tau)$ would give identical choice probabilities for every $c>0$.
The simplex constraint removes this ambiguity because
\begin{equation}\label{eq:app-beta-inverse}
\tau=\frac{1}{\mathbf 1^\top\boldsymbol{\beta}},
\qquad
\mathbf w=\frac{\boldsymbol{\beta}}{\mathbf 1^\top\boldsymbol{\beta}}.
\end{equation}
Thus, $(\mathbf w,\tau)\mapsto\boldsymbol{\beta}$ is a continuous bijection with a continuous inverse from $\Delta^{L-1}\times(0,\infty)$ onto $K\setminus\{\mathbf 0\}$, where $K=[0,\infty)^L$.

Identification from choices also depends on the reward contrasts.
Since the logistic sigmoid is strictly increasing, two parameter vectors $\boldsymbol{\beta}$ and $\boldsymbol{\beta}'$ give the same choice probabilities precisely when
$(\boldsymbol{\beta}-\boldsymbol{\beta}')^\top\mathbf d=0$ for the contrasts under consideration.
If those contrasts span $\mathbb R^L$, this implies $\boldsymbol{\beta}=\boldsymbol{\beta}'$, and \Cref{eq:app-beta-inverse} identifies $\mathbf w$ and $\tau$ separately.
For random contrasts, $\mathbb E[\mathbf d\mathbf d^\top]\succ0$ is a sufficient condition.

The temperature also accommodates different levels of choice consistency.
Because $|\mathbf w^\top\mathbf d_i|\le1$, fixing $\tau=1$ would restrict the probability of preferring $y_i^A$ to $[\sigma(-1),\sigma(1)]\approx[0.27,0.73]$.
Inferring $\tau$ permits more decisive probabilities without changing the normalization of the reward contrasts.

\subsection{Prior Specification}\label{app:prior}

BPM uses the independent priors $\mathbf w\sim\operatorname{Dir}(s\mathbf 1)$ and $\log\tau\sim\mathcal N(\log2,1)$, shared by personas within a task.
For $L\ge2$, each reward weight has a marginal distribution
\[
w_t\sim\operatorname{Beta}(s,(L-1)s),
\qquad
\mathbb E[w_t]=\frac1L,
\qquad
\operatorname{Var}(w_t)=\frac{L-1}{L^2(Ls+1)}.
\]
For $t\ne u$, the covariance is
\[
\operatorname{Cov}(w_t,w_u)=-\frac{1}{L^2(Ls+1)}.
\]
The total Dirichlet concentration is $Ls$.
Increasing $s$ reduces prior dispersion while preserving the uniform mean.

\paragraph{Variance matching.}
We set $s$ by matching the average marginal variance of a reference Dirichlet-tree prior.
The tree groups rewards using their contrasts on a separate calibration split.
Its construction uses average-linkage clustering with one minus the contrast correlation as the distance, followed by a strict-majority consensus over Bayesian-bootstrap input reweightings.
Each internal node assigns probability mass to its children using a symmetric Dirichlet distribution with unit parameters.

Let $q_T$ denote the resulting distribution over leaf weights and define
\[
v_T=\frac1L\sum_{t=1}^{L}\operatorname{Var}_{q_T}(w_t).
\]
Matching this quantity with the marginal variance of the symmetric prior gives
\begin{equation}\label{eq:app-variance-matching}
\frac{L-1}{L^2(Ls+1)}=v_T,
\qquad
s=\frac1L\left(\frac{L-1}{L^2v_T}-1\right),
\end{equation}
provided $0<v_T<(L-1)/L^2$.
The tree determines only $s$.
Posterior inference uses the flat prior $\operatorname{Dir}(s\mathbf 1)$ and does not retain the tree structure or its covariance pattern.

The reference moments can be computed from the independent node allocations.
Write $\boldsymbol{\rho}_v\sim\operatorname{Dir}(\boldsymbol{\alpha}_v)$ for the allocation at internal node $v$, with $\alpha_{v,c}=1$ for each child $c$ and $\alpha_{v,0}=\sum_c\alpha_{v,c}$.
The weight of leaf $t$ is the product of allocations along its path,
$w_t=\prod_{(v,c)\in\operatorname{path}(t)}\rho_{v,c}$.
Independence across nodes gives
\[
\mathbb E_{q_T}[w_t]
=\prod_{(v,c)\in\operatorname{path}(t)}
\frac{\alpha_{v,c}}{\alpha_{v,0}},
\qquad
\mathbb E_{q_T}[w_t^2]
=\prod_{(v,c)\in\operatorname{path}(t)}
\frac{\alpha_{v,c}(\alpha_{v,c}+1)}
{\alpha_{v,0}(\alpha_{v,0}+1)}.
\]
Subtracting the squared first moment yields the variance used in $v_T$.

\paragraph{Selection protocol.}
The variance-matching rule was fixed before the reported posterior fits, following a pilot on an earlier clinical expert library.
In that pilot, $s=1$ gave posterior means closer to uniform and poorer preference predictions on held-out comparisons.
The pilot evaluated two personas and did not construct personalized models.
For the reported tasks, calibration uses reward contrasts without persona labels and yields $s\approx0.199$, $0.247$, and $0.991$ for radiology report summarization, image captioning, and story generation, respectively.
These values remain fixed across personas and feedback budgets within each task.
The reported libraries were not evaluated under alternative values of $s$, so the pilot does not establish prior sensitivity for the final experiments.

\paragraph{Temperature prior and support.}
The prior on $\log\tau$ induces the density
\[
p(\tau)=\frac{1}{\tau\sqrt{2\pi}}
\exp\!\left[-\frac{(\log\tau-\log2)^2}{2}\right],
\qquad \tau>0.
\]
Together, the priors have full support on $\Delta^{L-1}\times(0,\infty)$ in the relative topology.
Although the Dirichlet prior assigns zero probability to the simplex boundary, every relative neighborhood of a boundary point has positive probability.
This support property is used in Proposition~\ref{prop:target}.

\subsection{Posterior Mean and Uncertainty}\label{app:posterior-summaries}
The posterior mean is a valid simplex vector because each component is nonnegative and
$\mathbf 1^\top\widehat{\mathbf w}=\mathbb E[\mathbf 1^\top\mathbf w\mid\mathcal D]=1$.
For any estimate $\mathbf a\in\Delta^{L-1}$,
\begin{equation}\label{eq:app-mean-loss}
\mathbb E[\|\mathbf w-\mathbf a\|_2^2\mid\mathcal D]
=\mathbb E[\|\mathbf w-\widehat{\mathbf w}\|_2^2\mid\mathcal D]
+\|\mathbf a-\widehat{\mathbf w}\|_2^2.
\end{equation}
To obtain this identity, expand $\mathbf w-\mathbf a=(\mathbf w-\widehat{\mathbf w})+(\widehat{\mathbf w}-\mathbf a)$ and use
$\mathbb E[\mathbf w-\widehat{\mathbf w}\mid\mathcal D]=\mathbf 0$ to eliminate the cross term.
Consequently,
\[
\widehat{\mathbf w}
=\operatorname*{arg\,min}_{\mathbf a\in\Delta^{L-1}}
\mathbb E[\|\mathbf w-\mathbf a\|_2^2\mid\mathcal D].
\]
This is an optimality statement for reward-weight estimation under squared loss.
Its connection to model optimality is established separately in Appendix~\ref{app:correspondence}.

\paragraph{Relative posterior spread.}
All coordinates have the same prior variance, so \Cref{eq:width} reduces to
\begin{equation}\label{eq:app-uncertainty}
U=\frac{\tfrac1L\sum_{t=1}^{L}\operatorname{sd}[w_t\mid\mathcal D]}
{\sqrt{(L-1)/(L^2(Ls+1))}}.
\end{equation}
Without feedback, the posterior equals the prior and $U=1$.
Smaller values indicate lower average marginal uncertainty relative to the prior.
The statistic summarizes the remaining uncertainty in reward weights under the preference model, and its relationship to personalization performance is examined in \S\ref{sec:analysis}.
The definition does not require $U\le1$ or a decrease after every additional comparison.

\subsection{Information Provided by Reward Contrasts}\label{app:identifiability}
The reward contrasts determine which directions in weight space enter the likelihood.
Condition on the observed contrasts and let
\[
H=\{\mathbf h\in\mathbb R^L:\mathbf 1^\top\mathbf h=0\},
\qquad
S=\operatorname{span}\{P_H\mathbf d_1,\dots,P_H\mathbf d_n\},
\]
where $P_H$ is the orthogonal projection onto the simplex tangent space $H$.

\begin{proposition}[Conditional invariance outside the contrast span]
\label{prop:span}
Write $\mathbf w=\mathbf 1/L+\mathbf u+\mathbf v$, where $\mathbf u\in S$ and $\mathbf v\in H\cap S^\perp$.
The likelihood in \Cref{eq:likelihood} depends on $\mathbf w$ only through $\mathbf u$.
For any prior on $\mathbf w$ independent of $\tau$, the conditional posterior satisfies
\[
p(\mathbf v\mid\mathbf u,\mathcal D)=p(\mathbf v\mid\mathbf u)
\]
almost surely, and $\dim S\le\min(n,L-1)$.
\end{proposition}

\begin{proof}
Since $\mathbf w-\mathbf 1/L\in H$, the orthogonal decomposition into $S$ and $H\cap S^\perp$ exists and is unique.
For each contrast,
\[
\mathbf w^\top\mathbf d_i
=\frac1L\mathbf 1^\top\mathbf d_i
+\mathbf u^\top P_H\mathbf d_i,
\]
because $\mathbf v$ is orthogonal to both $P_H\mathbf d_i$ and $\mathbf 1$.
Hence the likelihood is a function $g(\mathbf u,\tau)$, independent of $\mathbf v$.
Prior independence of $\mathbf w$ and $\tau$ gives
\[
p(\mathbf u,\mathbf v\mid\mathcal D)
\propto p(\mathbf u,\mathbf v)
\int g(\mathbf u,\tau)p(\tau)\,\mathrm d\tau.
\]
The integral depends only on $\mathbf u$ and cancels when conditioning on $\mathbf u$, proving the claimed equality.
Finally, $S$ is generated by $n$ vectors in the $(L-1)$-dimensional space $H$.
\end{proof}

This invariance is conditional on the informed component $\mathbf u$.
The marginal posterior is
\[
p(\mathbf v\mid\mathcal D)
=\int p(\mathbf v\mid\mathbf u)
\,p(\mathbf u\mid\mathcal D)\,\mathrm d\mathbf u,
\]
so it can change through prior dependence between $\mathbf u$ and $\mathbf v$.
The proposition identifies directions that the comparisons do not directly distinguish, rather than requiring every marginal distribution in those directions to remain unchanged.
Its dimension bound also distinguishes the number of comparisons from their geometric coverage.
Additional comparisons need not add independent directions, and the bound alone does not determine $U$ or a posterior contraction rate.

\subsection{Posterior Concentration}\label{app:posterior-concentration}
We next characterize the population target of Bayesian preference inference.
The result allows the Bradley--Terry model to be misspecified and identifies the generating parameters when Assumption~\ref{ass:bt} holds.

\begin{proposition}[Posterior concentration]
\label{prop:target}
Let the labeled comparisons $(\mathbf d_i,z_i)$ be independent and identically distributed, with $\|\mathbf d_i\|_\infty\le1$, $z_i\in\{-1,+1\}$, and $\mathbb E[\mathbf d\mathbf d^\top]\succ0$.
Suppose the expected log-loss
\[
R(\mathbf w,\tau)
=\mathbb E\!\left[-\log\sigma\!\left(
\frac{z\,\mathbf w^\top\mathbf d}{\tau}\right)\right]
\]
attains its minimum on $\Delta^{L-1}\times(0,\infty)$.
The minimizer $(\mathbf w^\circ,\tau^\circ)$ is unique.
Under any fixed probability prior with full support on this parameter space, the posterior concentrates at $(\mathbf w^\circ,\tau^\circ)$ almost surely as $n\to\infty$, and $\widehat{\mathbf w}\to\mathbf w^\circ$ almost surely.
Under Assumption~\ref{ass:bt}, the minimizer is $(\mathbf w^\star,\tau^\star)$.
\end{proposition}

\begin{proof}
Use the parameterization $\boldsymbol{\beta}=\mathbf w/\tau$ from Appendix~\ref{app:temperature-identification} and set $K=[0,\infty)^L$.
Define
\[
\begin{gathered}
\ell(\boldsymbol{\beta};\mathbf d,z)
=\log(1+\exp(-z\boldsymbol{\beta}^\top\mathbf d)),
\qquad
R(\boldsymbol{\beta})=\mathbb E[\ell(\boldsymbol{\beta};\mathbf d,z)],\\
R_n(\boldsymbol{\beta})=\frac1n\sum_{i=1}^{n}
\ell(\boldsymbol{\beta};\mathbf d_i,z_i).
\end{gathered}
\]
The likelihood is $\exp(-nR_n)$.
The assumed minimizer corresponds to a finite, nonzero vector $\boldsymbol{\beta}^\circ\in K$.

\emph{Uniqueness.}
The loss is finite and convex, with Hessian
\[
\nabla^2\ell(\boldsymbol{\beta};\mathbf d,z)
=\sigma(a)\sigma(-a)\,\mathbf d\mathbf d^\top,
\qquad a=z\boldsymbol{\beta}^\top\mathbf d.
\]
Bounded contrasts bound the gradient and Hessian, justifying differentiation under the expectation.
On each compact parameter set, $|a|\le\|\boldsymbol{\beta}\|_1$ implies a positive lower bound $c$ for $\sigma(a)\sigma(-a)$.
Thus $\nabla^2R\succeq c\,\mathbb E[\mathbf d\mathbf d^\top]\succ0$, so $R$ is strictly convex.
The minimizer on $K\setminus\{\mathbf 0\}$ also minimizes over $K$, because continuity and $t\boldsymbol{\beta}^\circ\in K\setminus\{\mathbf 0\}$ for $t>0$ give
$R(\mathbf 0)=\lim_{t\downarrow0}R(t\boldsymbol{\beta}^\circ)\ge R(\boldsymbol{\beta}^\circ)$.
Strict convexity therefore gives a unique minimizer on $K$.

\emph{Uniform convergence on compact sets.}
For fixed $\boldsymbol{\beta}$, the loss is bounded above by $\log2+\|\boldsymbol{\beta}\|_1$.
The strong law gives $R_n\to R$ almost surely on a countable dense subset of $\mathbb R^L$.
Moreover, the gradient bound $\|\nabla\ell\|_2\le\sqrt L$ makes both $R_n$ and $R$ uniformly Lipschitz.
A finite-net argument therefore extends convergence to uniform convergence on each compact set.
This is also a consequence of the convergence theorem for finite convex functions.

\emph{Concentration.}
Fix $\varepsilon>0$ and define
\[
\begin{aligned}
B_\varepsilon&=\{\boldsymbol{\beta}\in K:
\|\boldsymbol{\beta}-\boldsymbol{\beta}^\circ\|_2<\varepsilon\},\\
\Sigma_\varepsilon&=\{\boldsymbol{\beta}\in K:
\|\boldsymbol{\beta}-\boldsymbol{\beta}^\circ\|_2=\varepsilon\}.
\end{aligned}
\]
Continuity, uniqueness, and compactness of $\Sigma_\varepsilon$ imply
\[
\eta=\min_{\boldsymbol{\beta}\in\Sigma_\varepsilon}
\bigl(R(\boldsymbol{\beta})-R(\boldsymbol{\beta}^\circ)\bigr)>0.
\]
Choose $0<\delta<\min(\varepsilon,\|\boldsymbol{\beta}^\circ\|_2/2)$ such that
$R(\boldsymbol{\beta})<R(\boldsymbol{\beta}^\circ)+\eta/8$ on $B_\delta$.
Almost surely, uniform convergence gives $|R_n-R|<\eta/8$ on $\overline B_\varepsilon$ for all sufficiently large $n$.
Consequently,
\[
\begin{aligned}
R_n(\boldsymbol{\beta})&\ge R_n(\boldsymbol{\beta}^\circ)+3\eta/4
&&\text{on }\Sigma_\varepsilon,\\
R_n(\boldsymbol{\beta})&<R_n(\boldsymbol{\beta}^\circ)+3\eta/8
&&\text{on }B_\delta.
\end{aligned}
\]
For $\boldsymbol{\beta}\in K\setminus B_\varepsilon$, the segment from $\boldsymbol{\beta}^\circ$ to $\boldsymbol{\beta}$ meets $\Sigma_\varepsilon$ at
$\mathbf b=(1-a)\boldsymbol{\beta}^\circ+a\boldsymbol{\beta}$ for some $a\in(0,1]$.
Convexity extends the first bound to the entire exterior,
\[
R_n(\boldsymbol{\beta})-R_n(\boldsymbol{\beta}^\circ)
\ge\frac{R_n(\mathbf b)-R_n(\boldsymbol{\beta}^\circ)}{a}
\ge3\eta/4.
\]
Let $\Pi$ be the induced probability prior on $K\setminus\{\mathbf 0\}$, extended to $K$ with zero mass at the origin.
The continuous bijection in \Cref{eq:app-beta-inverse} preserves full support, so $\Pi(B_\delta)>0$.
The posterior mass outside $B_\varepsilon$ satisfies
\[
\begin{aligned}
\Pi(K\setminus B_\varepsilon\mid\mathcal D)
&=\frac{\int_{K\setminus B_\varepsilon}e^{-nR_n}\,\mathrm d\Pi}
{\int_K e^{-nR_n}\,\mathrm d\Pi}\\
&\le\frac{e^{-n(R_n(\boldsymbol{\beta}^\circ)+3\eta/4)}}
{\Pi(B_\delta)e^{-n(R_n(\boldsymbol{\beta}^\circ)+3\eta/8)}}
=\frac{e^{-3n\eta/8}}{\Pi(B_\delta)}\longrightarrow0.
\end{aligned}
\]
Continuity of the inverse parameterization gives posterior concentration at $(\mathbf w^\circ,\tau^\circ)$.
Since $\mathbf w$ is bounded on the simplex, its posterior mean converges to $\mathbf w^\circ$ as well.

\emph{Identification of the target.}
Let $p_\star(\mathbf d)=\Pr(z=+1\mid\mathbf d)$ be the persona's choice probability and $q_{\boldsymbol{\beta}}(\mathbf d)=\sigma(\boldsymbol{\beta}^\top\mathbf d)$.
Expected log-loss decomposes as
\[
R(\boldsymbol{\beta})
=\mathbb E_{\mathbf d}\!\left[
\mathrm{KL}\bigl(\mathrm{Bern}(p_\star(\mathbf d))
\,\|\,\mathrm{Bern}(q_{\boldsymbol{\beta}}(\mathbf d))\bigr)
\right]+C,
\]
where $C$ is the expected conditional entropy, independent of $\boldsymbol{\beta}$.
The minimizer therefore gives the closest choice model in expected Kullback--Leibler divergence, consistent with the interpretation of posterior targets under misspecification.
Under Assumption~\ref{ass:bt}, $p_\star(\mathbf d)=q_{\boldsymbol{\beta}^\star}(\mathbf d)$ for $\boldsymbol{\beta}^\star=\mathbf w^\star/\tau^\star$.
The divergence vanishes at $\boldsymbol{\beta}^\star$, and uniqueness identifies the minimizer with the generating parameters.
\end{proof}

\paragraph{Interpretation and sampling scope.}
When the preference model is misspecified, $(\mathbf w^\circ,\tau^\circ)$ is defined relative to the distribution of reward contrasts supplied to inference.
The elicitation pool therefore affects the population target as well as the information available about it.
The proposition describes an idealized sequence of independent comparisons.
Comparisons drawn from a finite shared pool may share inputs or outputs, so this result does not directly establish consistency for that experimental sampling protocol or quantify error at the reported feedback budgets.

\subsection{Connection to Model Optimality}\label{app:correspondence}
BPM sets the merge coefficients to the inferred posterior mean.
Rewarded Soups motivates this direct assignment as an empirical strategy, while its Pareto-coverage argument concerns the existence of a suitable interpolation~\citep[Section~2.2.3]{rame2023rewardedsoups}.
Assumption~\ref{ass:coord} states the additional correspondence used here to interpret the limiting BPM model as optimal within the merge family.

Under Assumption~\ref{ass:bt} and the conditions of Proposition~\ref{prop:target},
\[
\boldsymbol{\lambda}=\widehat{\mathbf w}
\xrightarrow{\mathrm{a.s.}}\mathbf w^\star,
\qquad
\boldsymbol{\theta}(\boldsymbol{\lambda})
\xrightarrow{\mathrm{a.s.}}\boldsymbol{\theta}(\mathbf w^\star).
\]
The second convergence follows because the expert library is fixed and \Cref{eq:family} is continuous in the merge coefficients.
For LoRA, the effective matrices in \Cref{eq:app-factor-merge} are continuous in those coefficients as well.
By Assumption~\ref{ass:coord}, the limiting model maximizes
\[
\mathbb E_x\mathbb E_{y\sim\pi_{\boldsymbol{\theta}(\boldsymbol{\lambda})}(\cdot\mid x)}
\left[\mathbf w^{\star\top}\tilde{\mathbf r}(x,y)\right]
\]
over $\boldsymbol{\lambda}\in\Delta^{L-1}$.
Thus, accurate preference inference and the correspondence assumption together yield convergence in model parameters to a maximizer of the persona's expected weighted reward.
The conclusion concerns the limit as feedback increases and gives no finite-sample optimality guarantee.
BPM still constructs a model when the correspondence does not hold, and \S\ref{sec:analysis} examines the behavior of direct weighting empirically.

\section{Experimental Details}\label{app:setup}

\subsection{Data Roles and Separation}\label{app:data}
Table~\ref{tab:splits} lists the data assigned to each experimental role.
For MIMIC-CXR, the expert-training, calibration, and remaining partitions are patient-disjoint, with the last partition divided into separate reports for pool simulation, reward standardization, and testing.
We further exclude $49$ test reports whose patients occur in pool simulation, leaving $451$ reports from $407$ patients.
COCO separates the roles by image, and WritingPrompts by cleaned prompt text; transfer analyses use additional inputs without refitting the posterior.

Prior calibration uses reward contrasts without persona labels (Appendix~\ref{app:prior}).
The reward means and standard deviations in Eq.~\eqref{eq:contrast} are estimated from separate standardization inputs, responses from merge-family models on the standardization inputs.

\begin{table}[t]
\centering\small
\setlength{\tabcolsep}{5pt}
\begin{tabular}{lrrr}
\toprule
Data role & \textsc{Summarization} & \textsc{Image Captioning} & \textsc{Story Generation}\\
\midrule
Expert-training pool & 85.6k reports & 600 images & 1,000 prompts\\
Prior calibration & 1,000 patients & 1,000 images & 1,000 prompts\\
Pool simulation & 250 reports & 250 images & 250 prompts\\
Reward standardization & 250 reports & 250 images & 250 prompts\\
Test & 451 reports & 500 images & 500 prompts\\
\bottomrule
\end{tabular}
\caption{Data splits. Expert-training inputs are sampled from the available data shown here, with training sample counts given in Table~\ref{tab:training-config}.}
\label{tab:splits}
\end{table}

\subsection{Reward Dimensions and Expert Libraries}\label{app:rewards}
\paragraph{\textsc{Summarization} ($29$ experts).}
The rewards comprise ten text-similarity metrics (BLEU-1--4, ROUGE-1/2/L, METEOR, chrF, and BERTScore); nine clinical metrics (RadGraph F1 and its anatomy, observation, and relation components; CheXbert precision, recall, and F1; RaTEScore; and GREEN~\citep{ostmeier2024green}); five general summarization metrics (SummaC, AlignScore, BLEURT, MoverScore, and BARTScore); and five form-related metrics (compression, coverage, density, distinct-$n$, and novel-$n$-gram scores).
The form-related rewards measure agreement with the reference impression's corresponding properties.

\paragraph{\textsc{Image Captioning} ($21$ experts).}
The rewards comprise BLEU-1--4, ROUGE-L, METEOR, and CIDEr; five SPICE components (object, attribute, relation, color, and count); CLIPScore and RefCLIPScore; overall SPICE precision, recall, and F1; object precision, recall, and F1 against COCO instance annotations using the CHAIR vocabulary~\citep{rohrbach2018chair}; and Flesch--Kincaid grade.

\paragraph{\textsc{Story Generation} ($16$ experts).}
Five rewards measure length, distinct bigrams, negative self-BLEU, Flesch--Kincaid grade, and negative repetition rate directly from the text.
Six HANNA criteria~\citep{chhun2022hanna} measure relevance, coherence, empathy, surprise, engagement, and complexity.
Five further style criteria measure imagery, prose elegance, showing rather than telling, avoidance of purple prose, and dialogue quality.
The eleven rubric rewards use gpt-oss-20B with the four-level anchors in Appendix~\ref{app:rubric-anchors}, reversing negatively oriented criteria so that higher rewards are better.
The persona-based judges instead use Qwen3.5-27B for feedback and test evaluation.

\subsection{Expert Training and Compute}\label{app:training}
\paragraph{LoRA configuration.}
All experts use standard LoRA with rank $r=16$, scaling parameter $\alpha=32$, and zero dropout, keeping base weights and bias terms frozen.

\paragraph{GRPO and RLOO.}
The GRPO configuration in Table~\ref{tab:training-config} uses batch-level reward scaling with KL coefficient $0$.
\textsc{Story Generation} additionally masks truncated completions and uses token-level truncated importance sampling for the rollout-policy correction.
The \textsc{Summarization} RLOO library uses the same learning rate, training steps, and rollout budget, with leave-one-out advantage centering and batch-level normalization.

\paragraph{DPO experts.}
The additional \textsc{Summarization} library uses reward-labeled pairs obtained by selecting the highest- and lowest-scoring responses among eight base-policy samples per input, excluding inputs without two scored responses or a positive score difference.
The DPO configuration uses one epoch, learning rate $3\times10^{-5}$, $\beta=0.1$, and an effective batch of $32$ pairs, with the frozen base model as the reference.

\begin{table}[t]
\centering\small
\setlength{\tabcolsep}{5pt}
\begin{tabular}{lrrr}
\toprule
Setting & \textsc{Summarization} & \textsc{Image Captioning} & \textsc{Story Generation}\\
\midrule
Training inputs per expert & 1,600 & 600 & 1,000\\
Optimizer steps & 100 & 100 & 100\\
Learning rate & $3\times10^{-5}$ & $3\times10^{-5}$ & $3\times10^{-5}$\\
Completions per prompt & 8 & 8 & 8\\
Completions per optimizer step & 128 & 128 & 128\\
Sampling temperature & 1.0 & 1.0 & 0.7\\
Top-$p$ & 1.0 & 1.0 & 0.95\\
Maximum completion tokens & 1,024 & 128 & 2,048\\
\bottomrule
\end{tabular}
\caption{Training configuration for the main GRPO expert libraries.}
\label{tab:training-config}
\end{table}

\paragraph{Compute resources.}
Policy training uses NVIDIA A100 80GB GPUs, while posterior inference runs on CPUs.

\subsection{Pool Construction and Feedback Simulation}\label{app:pool-construction}
The $20$ random merges used to construct pool $\mathcal C$ have coefficients sampled from $\operatorname{Dir}(\mathbf 1)$.
For each input, we generate one response per model and sample four pairs, each comprising responses from distinct models.
These pairs are shared across personas, and both responses are scored on all reward dimensions.

\paragraph{Preference assignment and filtering.}
For prompt-based personas, both presentation orders must yield valid, non-tied verdicts identifying the same preferred response.
Rubric-based comparisons exclude ties and responses with missing required scores (Appendix~\ref{app:forma}).
Feedback subsets are sampled after this persona-specific filtering, so $n$ counts the comparisons supplied to each method.

\paragraph{Generation and judgment settings.}
\textsc{Summarization} and \textsc{Image Captioning} decode greedily; \textsc{Story Generation} samples at temperature $0.7$ and top-$p$ $0.95$ with a $2{,}048$-token completion budget.
For prompt-based personas, the judging model uses temperature $0$, disabled thinking, and a maximum of $2{,}048$ generated tokens.

\paragraph{Retained feedback and elicitation cost.}
The reported budget $n$ is the number of retained pairwise comparisons used for fitting, sampled from a preconstructed filtered pool.
For a prompt-based persona, let $M$ be the number of candidate pairs judged, $R$ the number retained after filtering, and $J$ the number of order-specific judgments, including retries if any.
The pool acceptance rate is $R/M$, and its average elicitation cost is $J/R$ judgments per retained comparison.
With exactly two judgments per pair and no retries, this cost is $2M/R$.
These are pool-construction costs: the retained-budget sweep reuses the pool and does not measure an online user-query budget.

\subsection{Compared Method Configurations}\label{app:baselines}
\paragraph{Ridge merge and reranking.}
\textsc{rm} fits logistic regression to standardized reward contrasts before the
per-pair $\ell_\infty$ normalization used by \bpm, with an $\ell_2$
penalty of one, and applies softmax to the fitted coefficients for merging.

\paragraph{Per-persona DPO.}
Using the same preference subsets as \textsc{bpm}, \textsc{dpo} trains a fresh adapter with the task-specific LoRA configuration, $40$ optimizer steps, learning rate $3\times10^{-5}$, $\beta=0.1$.
It uses a microbatch of four pairs with eight gradient-accumulation steps and the frozen base model as reference.

\paragraph{\textsc{icai}.}
The constitution is inferred by gpt-oss-20b and is shared between \textsc{icai-b} and \textsc{icai-m}.

\paragraph{Language-to-Weights.}
At temperature $0$, gpt-oss-20b receives the task description, persona prompt, and a one-line definition of each reward, then allocates $100$ points across reward dimensions.
These weights are normalized and used as merge coefficients.

\subsection{Posterior Inference and Randomization}\label{app:seeds}
\paragraph{Posterior inference.}\label{app:inference}
We use NUTS with four chains, each having $1{,}000$ warmup iterations and $2{,}000$ retained samples, requiring split $\widehat R\le1.01$ and a sample size of at least $400$ for each weight.

\paragraph{Training seeds and feedback subsets.}
The \textsc{Summarization} robustness analysis includes GRPO expert libraries trained with two additional seeds.
For each of three feedback repetitions indexed by $d$, we shuffle each persona's filtered pool using seed $1000d$ and take the first $n$ comparisons.
This produces nested subsets across budgets while keeping the expert library fixed.
Posterior sampling uses seed $d$, with seed $d+1000$ for a single retry if convergence criteria are not met; the same criteria apply to the retry.
Unless otherwise specified, main-text results are averaged over three sampled feedback subsets.

\subsection{Test Metrics and Statistical Aggregation}\label{app:ci}
\paragraph{Win rates and aggregation.}
For a fixed task, comparison method, and feedback budget, let $\mathrm{WR}_{p,d}$ be the decided test win rate for persona $p$ and feedback subset $d$, excluding ties, order-inconsistent verdicts, and invalid or missing judgments.
Models fitted to different feedback subsets are evaluated on the same test inputs, and the macro win rate is
\begin{equation}
\mathrm{WR}_{\mathrm{macro}}
=\frac{1}{P}\sum_{p=1}^{P}\left(\frac{1}{D}\sum_{d=0}^{D-1}\mathrm{WR}_{p,d}\right),
\label{eq:setup-macro}
\end{equation}
where $P$ is the number of personas and $D$ the number of feedback subsets included in the comparison.
Personas receive equal weight, while each win rate is conditional on a valid, decided judgment.

All results in Table~\ref{tab:main} use $D=3$.

\paragraph{95\% Confidence intervals.}
We recompute \Cref{eq:setup-macro} over $2{,}000$ bootstrap samples of test inputs, resampled with replacement using shared indices across personas and feedback subsets.
The $2.5$th and $97.5$th percentiles give the $95\%$ confidence interval, quantifying test-input uncertainty conditional on the fixed feedback subsets and expert library.
These intervals quantify test-input uncertainty for the observed personas and fitted models.
They do not include variation from newly sampled feedback, independently trained expert libraries, or alternative persona and judge specifications.
Feedback repetitions may overlap because they sample the same finite filtered pool; budgets are nested within each repetition.

\paragraph{Accounting for undecided judgments.}
We additionally report wins, losses, ties, order-inconsistent verdicts, and invalid judgments, together with a supplementary score that retains valid undecided pairs.
Because the separate presentation-order decisions were not retained, the score assigns a central value to order-inconsistent pairs and reports bounds for their possible contributions.
The construction and results appear in Appendix~\ref{app:evaluation-accounting}.

\section{Persona Specifications}\label{app:personas}
The following descriptions and comparison instructions define the six, six, and seven main personas for \textsc{Summarization}, \textsc{Image Captioning}, and \textsc{Story Generation}, respectively, and remain fixed between pool simulation and test evaluation.
The additional captioning persona \emph{Annotator} is used only in diagnostic analyses and is excluded from the main macro averages.


\subsection{Summarization Prompts}\label{app:prompts-mimic}

\paragraph{Common instruction.}

\begin{quote}\small
The IMPRESSION section is a summary written from the FINDINGS section. You are judging two candidate impressions produced from the same findings. The persona description below defines the intended use and the trade-offs you should apply. Apply those criteria consistently regardless of whether a candidate is labelled A or B.
\end{quote}

\paragraph{Default instruction.}

\begin{quote}\small
The IMPRESSION section of a radiology report is a summary of the FINDINGS section: it must convey the clinically important content of the findings without restating them in full. You are a radiologist judging two candidate summaries of the same findings.
\end{quote}

\paragraph{Comparison template.}

\begin{quote}\small
FINDINGS:
\{source\}

SUMMARY A:
\{a\}

SUMMARY B:
\{b\}

Which is the better summary of these findings under your stated criteria? Choose TIE only if the two are effectively equivalent under those criteria; otherwise choose one even if the difference is small. Do not restate the candidate summaries in your reasoning. Give at most two sentences of reasoning, then end with a line of exactly this form:
VERDICT: A     (or)     VERDICT: B     (or)     VERDICT: TIE
\end{quote}

\paragraph{Ward.}

\begin{quote}\small
You are the ward clinician who will act on this impression. Check each summary against the findings for what it omits: new or changed findings, support devices and their position, and relevant negatives that affect inpatient management. Prefer the summary that omits less of that content; missing information is worse than extra length. Reject any diagnosis the findings do not support, however plausible. When both cover the management-relevant content, brevity is welcome.
\end{quote}

\paragraph{ED.}

\begin{quote}\small
You are the emergency clinician reading this impression during a busy shift. You want the single most urgent, actionable conclusion stated first and as briefly as possible. Whenever both summaries state the urgent conclusion, prefer the shorter one; every extra sentence costs you reading time. A summary that buries or omits the urgent finding is seriously flawed, as is one that asserts an unsupported diagnosis. Detail that would not change acute management is a cost rather than a benefit.
\end{quote}

\paragraph{Complete.}

\begin{quote}\small
You judge a summary by how much of the report's clinically relevant content it retains, including negative and incidental findings. A summary that drops a relevant finding has failed, even if it is short and reads well. Length is not a fault.
\end{quote}

\paragraph{Concise.}

\begin{quote}\small
You judge a summary by how concisely it states the key diagnosis. A summary that restates findings or adds anything beyond the main conclusion has not summarised, however accurate it is. Omitting minor findings is not a fault.
\end{quote}

\paragraph{Faithful.}

\begin{quote}\small
You judge a summary by whether every statement in it is supported by the findings. Any claim, diagnosis or inference that the findings do not state is a failure, however plausible. Length and completeness are not what you judge.
\end{quote}

\paragraph{Calibrated.}

\begin{quote}\small
You judge a summary by whether it preserves the certainty level of the findings: possible or probable findings must stay qualified, definite ones stated as such. Overstated certainty is a failure. Length and completeness are not what you judge.
\end{quote}

\subsection{Image Captioning Prompts}\label{app:prompts-coco}

\paragraph{Common instruction.}

\begin{quote}\small
You are judging two one-sentence captions written for the same photograph, which you can see. The persona description below defines the intended use and the trade-offs you should apply. Apply those criteria consistently regardless of whether a candidate is labelled A or B.
\end{quote}

\paragraph{Default instruction.}

\begin{quote}\small
A caption is a single-sentence description of an image, written for someone who cannot see it. You are looking at the image and judging two candidate captions of it.
\end{quote}

\paragraph{Comparison template.}

\begin{quote}\small
CAPTION A:
\{a\}

CAPTION B:
\{b\}

Which is the better caption for this image under your stated criteria? Choose TIE only if the two are effectively equivalent under those criteria; otherwise choose one even if the difference is small. Do not restate the candidate captions in your reasoning. Give at most two sentences of reasoning, then end with a line of exactly this form:
VERDICT: A     (or)     VERDICT: B     (or)     VERDICT: TIE
\end{quote}

\paragraph{Default.}

No additional persona description is supplied; the default instruction above is used.

\paragraph{Indexer.}

\begin{quote}\small
You catalogue images for visual search. You prefer the caption that specifically names more of the objects that are visible. When attributes, colours, and counts are clear, you want them included. You also value the setting when it would help someone retrieve the image. You treat every omission of something clearly visible as a fault, but you do not penalize a caption for being long.
\end{quote}

\paragraph{Magazine Editor.}

\begin{quote}\small
You edit a photography magazine and appreciate fine prose. You prefer a caption that forms a well-crafted sentence with vivid, precise, and varied vocabulary. You value an evocative sense of light, texture, and atmosphere, along with graceful rhythm. You treat plain, inventory-like wording as a fault. You welcome length when every word earns its place. You still require the caption to describe this image rather than a different one.
\end{quote}

\paragraph{Verifier.}

\begin{quote}\small
You edit captions before publication and prefer the caption whose every statement is directly supported by what you can see. You treat claims about people's feelings or intentions, the scene's location, events before or after the pictured moment, and anything outside the frame as faults, however plausible they seem. You also penalize any object, colour, or count that is not shown. You do not judge coverage, length, or style.
\end{quote}

\paragraph{Picture Book.}

\begin{quote}\small
You choose captions for a picture book aimed at young children. You prefer one complete sentence that a six-year-old can read aloud and understand. It uses common everyday words, stays short and concrete, and says what is happening. You treat rare or long words, sub-clauses, lists, and abstract language as faults. You do not penalize a caption for leaving out minor details in the background.
\end{quote}

\paragraph{Alt-text.}

\begin{quote}\small
You write alt text for someone using a screen reader. You prefer one short, plain sentence that identifies the main subject, describes its most important visible action, and includes only the context needed to understand the image. You penalize speculation, decorative language, and background details that do not change the meaning. You consider a caption unsuccessful if it omits the main subject or its action.
\end{quote}

\paragraph{Annotator (additional diagnostic persona).}
Excluded from the main macro averages.
\begin{quote}\small
You annotate images for a COCO-style captioning dataset and follow its annotation guidelines. You prefer one plain sentence of roughly eight to fifteen words that names the important objects, says what they are doing, uses ordinary words, and avoids proper names. You treat guesses about the past, the future, feelings, or anything outside the frame as faults, along with decorative wording and minor background details. You do not penalize a longer or rarer word when it is the accurate one.
\end{quote}

\subsection{Story Generation Prompts}\label{app:prompts-story}

\paragraph{Common instruction.}

\begin{quote}\small
You are shown a writing prompt and two stories written in response to it. The persona description below defines the role you read in and the trade-offs you should apply. Apply those criteria consistently regardless of whether a story is labelled A or B.
\end{quote}

\paragraph{Default instruction.}

\begin{quote}\small
You are shown a writing prompt and two stories written in response to it.
\end{quote}

\paragraph{Comparison template.}

\begin{quote}\small
WRITING PROMPT:
\{source\}

STORY A:
\{a\}

STORY B:
\{b\}

Which story would you, in this role, rather publish or read? Choose TIE only if the two are effectively equivalent under your stated criteria; otherwise choose one even if the difference is small. Do not restate the stories in your reasoning. Give at most two sentences of reasoning, then end with a line of exactly this form:
VERDICT: A     (or)     VERDICT: B     (or)     VERDICT: TIE
\end{quote}

\paragraph{Literary.}

\begin{quote}\small
You are the fiction editor of a literary magazine. You value restraint and precision: concrete images that do the emotional work, characters revealed through what they do and notice rather than through statements about how they feel, and endings that trust the reader. Clichés, melodrama, decoration without purpose, and sentences that explain what the story has already shown all count against a story. Extra length earns no credit by itself, and you would rather publish a restrained, emotionally suggestive story with little action than a brisk, neatly resolved one whose prose is less precise.
\end{quote}

\paragraph{Genre.}

\begin{quote}\small
You are an acquiring editor for a commercial genre imprint (thriller, fantasy, science fiction, romance). You want a story that hooks in the first lines, makes the central problem clear, moves through consequential action and dialogue, and resolves that problem with an earned payoff. You prefer forward action and a clear resolution to sustained interior reflection and an ending left open to interpretation. You would rather acquire a brisk, clearly resolved story in plain, serviceable prose than a beautifully phrased, reflective one with little forward action.
\end{quote}

\paragraph{Librarian.}

\begin{quote}\small
You are a children's librarian choosing stories to read aloud to six- to nine-year-olds. You want clear, mostly familiar language that reads naturally aloud, sympathetic characters a child can root for, and a warm, reassuring ending. Frightening or adult material and persistently difficult wording count strongly against a story, while brief, understandable moments of danger or sadness can work within a reassuring whole.
\end{quote}

\paragraph{Flash.}

\begin{quote}\small
You judge flash fiction and prefer compact stories, usually around 300 words or fewer. You favour a tightly focused situation, a memorable central image, and an ending that changes how the reader sees what came before. You value compression, but a slightly longer story can win when its extra words make the image or the final turn more effective.
\end{quote}

\paragraph{Screenwriter.}

\begin{quote}\small
You are a screenwriter who enjoys short fiction built around dramatic scenes. You favour concrete scenes where distinctive dialogue and visible action reveal conflict and character, with thoughts and backstory kept brief. You prefer conflict enacted between characters to extended interior monologue or descriptive passages that interrupt a scene.
\end{quote}

\paragraph{Restorative.}

\begin{quote}\small
You read fiction from a restorative perspective on wrongdoing. When a story deals with harm, you favour one that treats accountability as repair — the wrongdoer understanding what they did, the harmed person having a say, some restoration attempted, however partial — over one that resolves through punishment, revenge, or the wrongdoer simply getting what was coming to them. This holds even when the punitive story is equally coherent, equally moving, and equally well written. Where a story involves no wrongdoing, judge it on whether it takes its human situation seriously.
\end{quote}

\paragraph{Default.}

No additional persona description is supplied; the default instruction above is used.

\section{Rubric-Based Personas}\label{app:forma}
Table~\ref{tab:rubric-personas} defines the five rubric-based \textsc{Story Generation} personas.
Qwen3.5-27B scores the eleven criteria through AutoRubric~\citep{rao2026autorubric} with temperature $0$ and thinking disabled, using the anchors below to map each criterion to $\{0,1/3,2/3,1\}$ with higher scores preferred.
Brevity and readability use word-count thresholds of $300$, $450$, and $600$ and Flesch--Kincaid thresholds of $4$, $6$, and $8$, respectively, assigning higher scores to shorter and more readable text.

\begin{table}[t]
\centering\small
\begin{tabular}{lp{0.72\textwidth}}
\toprule
Persona & Nonzero criterion weights\\
\midrule
\emph{Literary} & Prose elegance $0.30$; showing rather than telling $0.25$; imagery $0.25$; avoidance of purple prose $0.20$.\\
\emph{Genre} & Engagement $0.35$; surprise $0.25$; dialogue quality $0.20$; coherence $0.20$.\\
\emph{Librarian} & Readability $0.35$; coherence $0.25$; empathy $0.25$; avoidance of purple prose $0.15$.\\
\emph{Flash} & Brevity $0.35$; surprise $0.35$; imagery $0.30$.\\
\emph{Screenwriter} & Dialogue quality $0.45$; showing rather than telling $0.35$; imagery $0.20$.\\
\bottomrule
\end{tabular}
\caption{Weights defining the five rubric-based \textsc{Story Generation} personas. Unlisted criteria have zero weight.}
\label{tab:rubric-personas}
\end{table}

A persona prefers the response with the larger weighted score over its nonzero-weight criteria; pairs with ties or missing required scores are excluded without renormalizing the remaining weights.
The controlled preference sweep~\Cref{fig:analysis}b interpolates the \emph{Literary} and \emph{Genre} weight vectors at $t\in\{0.17,0.33,0.50,0.67,0.83\}$, using $(1-t)\mathbf h_{\mathrm{Literary}}+t\mathbf h_{\mathrm{Genre}}$ to generate feedback.
The endpoint rubrics remain fixed when evaluating the resulting stories.

\subsection{Criterion Definitions and Score Anchors}\label{app:rubric-anchors}

The following four-level anchors specify the story-quality criteria. For positively oriented criteria, reward is the level divided by three. For fault criteria, reward is one minus the level divided by three. The same definitions and anchors are supplied to the reward scorer and the persona criterion scorer.

\paragraph{relevance.}

How well the story matches the writing prompt.

\begin{enumerate}[label=\arabic*.,start=0,leftmargin=*,itemsep=0pt]

\item The story ignores the prompt: none of the prompt's premise, characters or situation appears in it.

\item The story touches the prompt only loosely, picking up a word or a mood but not the situation the prompt sets up.

\item The story clearly works from the prompt but drops or contradicts part of its premise.

\item The story takes on the prompt's whole premise and stays inside it from beginning to end.

\end{enumerate}

\paragraph{coherence.}

How much the story makes sense.

\begin{enumerate}[label=\arabic*.,start=0,leftmargin=*,itemsep=0pt]

\item The story does not make sense: sentences do not connect, or characters and events change without explanation.

\item Individual passages read, but the story as a whole is confusing, with events or characters that contradict each other.

\item The story mostly follows, with one or two jumps, loose ends or inconsistencies a reader has to work around.

\item Every event and character follows from what came before; nothing in the story has to be explained away.

\end{enumerate}

\paragraph{empathy.}

How well the reader understood the characters' emotions.

\begin{enumerate}[label=\arabic*.,start=0,leftmargin=*,itemsep=0pt]

\item The characters have no discernible emotional life; there is nothing for the reader to understand.

\item Emotions are only asserted or named, so the reader is told how a character feels without being able to feel it.

\item The reader understands what the characters feel at the story's main turns, though not consistently.

\item The reader follows the characters' feelings throughout, grounded in what the characters do, say and notice.

\end{enumerate}

\paragraph{surprise.}

How surprising the end of the story was.

\begin{enumerate}[label=\arabic*.,start=0,leftmargin=*,itemsep=0pt]

\item The ending is entirely predictable, or the story simply stops without arriving anywhere.

\item The ending is the obvious continuation of the premise, with nothing the reader had not already assumed.

\item The ending contains a turn the reader did not expect, though a mild or familiar one.

\item The ending genuinely surprises and still fits the story, changing how the reader sees what came before.

\end{enumerate}

\paragraph{engagement.}

How much the reader engaged with the story.

\begin{enumerate}[label=\arabic*.,start=0,leftmargin=*,itemsep=0pt]

\item The story gives the reader no reason to continue; reading it is an effort.

\item The story holds attention in places but stalls, repeats itself or loses the reader's interest.

\item The story keeps the reader reading, though attention slackens in parts of it.

\item The story holds the reader from the first lines to the last and makes them want to know what happens.

\end{enumerate}

\paragraph{complexity.}

How elaborate the story is.

\begin{enumerate}[label=\arabic*.,start=0,leftmargin=*,itemsep=0pt]

\item The story is a bare sketch: one flat situation, no development of character, plot or theme.

\item The story has a simple single thread with little development beyond stating what happens.

\item The story develops a plot, a character or an idea with some depth, but along a single dimension.

\item The story is elaborate, developing plot, characters and theme together and holding them in relation.

\end{enumerate}

\paragraph{Imagery \& Descriptive Quality.}

How vividly and concretely the story renders what can be seen, heard, touched, smelled and tasted.

\begin{enumerate}[label=\arabic*.,start=0,leftmargin=*,itemsep=0pt]

\item The story is abstract throughout: nothing in it can be pictured, and no sense beyond sight is engaged even in outline.

\item Description is generic -- rooms, faces and weather named but not rendered, with images that could belong to any story.

\item The story contains some concrete, specific images that a reader can see, though much of it stays at the level of summary.

\item The story is rendered in precise, particular sensory detail that puts the reader in the scene, with images chosen rather than defaulted to.

\end{enumerate}

\paragraph{Elegant Prose.}

How well-made the sentences are: rhythm, control of syntax, word choice, and variation.

\begin{enumerate}[label=\arabic*.,start=0,leftmargin=*,itemsep=0pt]

\item The prose is clumsy: broken or monotonous syntax, repeated sentence shapes, and words that do not mean what the sentence needs.

\item The prose is serviceable but flat -- correct sentences with no rhythm, no variation in length or shape, and predictable diction.

\item The prose is controlled and readable, with some variation of rhythm and some well-chosen words, but also passages that go slack.

\item The prose is consistently well made: sentences vary in shape and rhythm to a purpose, and the diction is precise throughout.

\end{enumerate}

\paragraph{Tell-Don't-Show.}

How much the story states its content outright instead of dramatising it. This is a FAULT: score how badly the story tells rather than shows.

\begin{enumerate}[label=\arabic*.,start=0,leftmargin=*,itemsep=0pt]

\item The story dramatises throughout: what characters feel and what is at stake reaches the reader through action, speech and detail.

\item The story mostly dramatises, with one or two places where it stops to state outright something it has already shown.

\item The story tells often -- feelings, motives and significance are announced to the reader as much as they are enacted.

\item The story is told rather than dramatised: it is a report of what happened and what it meant, with little or nothing enacted on the page.

\end{enumerate}

\paragraph{Purple Prose.}

How overwrought the writing is: ornamental adjectives, strained metaphor, and grandiosity out of proportion to the material. This is a FAULT: score how purple the prose is.

\begin{enumerate}[label=\arabic*.,start=0,leftmargin=*,itemsep=0pt]

\item The writing is unadorned; every image and modifier earns its place.

\item There are occasional ornamental flourishes or an over-reached metaphor, but the prose is not carried away.

\item The writing is frequently overwrought: stacked adjectives, strained figures and a register grander than the material.

\item The writing is relentlessly purple -- ornament substitutes for substance and the prose is exhausting to read.

\end{enumerate}

\paragraph{Weak Dialogue.}

How poorly the spoken lines work: stiff or interchangeable voices, exposition delivered as speech, and dialogue that does no work in the scene. This is a FAULT: score how weak the dialogue is. A story with no dialogue at all cannot be strong on this criterion and is scored 2.

\begin{enumerate}[label=\arabic*.,start=0,leftmargin=*,itemsep=0pt]

\item The dialogue is alive: distinct voices, lines that carry the scene, and nothing said only for the reader's benefit.

\item The dialogue mostly works, with a few stiff lines or a passage where a character explains something for the reader.

\item The dialogue is weak: voices are interchangeable, lines are stilted, or speech is used chiefly to deliver exposition. A story with no dialogue is scored here.

\item The dialogue is unusable -- wooden, uniform and purely expository -- or the story is one where speech was needed and none of it convinces.

\end{enumerate}

\section{Additional Results}\label{app:results}

\subsection{Persona-Level Performance}\label{app:full}
\Cref{fig:heatmap-persona} expands the main comparisons at $n=500$.
Each cell averages three feedback subsets, using the same persona sets and decided-win-rate definition as Table~\ref{tab:main}.

\begin{figure}[tbp]
\centering
\includegraphics[width=\textwidth]{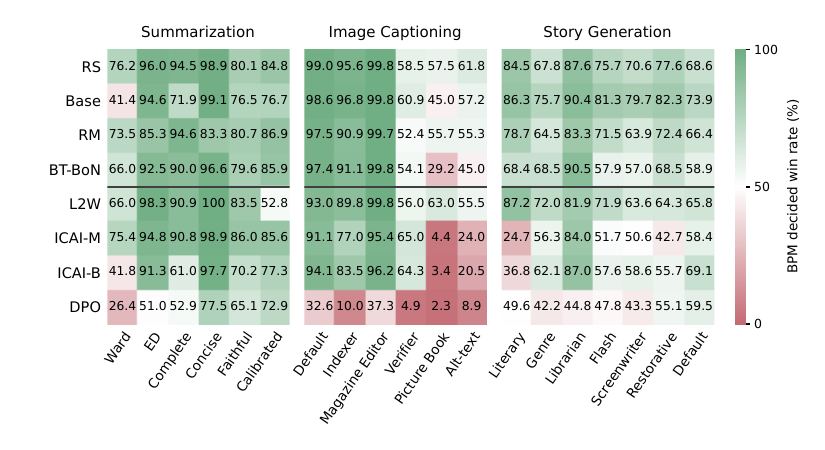}
\caption{\textbf{Persona-level personalization performance at $n=500$.}
Each cell is \bpm's decided win rate (\%) against the row's compared method, averaged over three feedback subsets.
Colors follow Table~\ref{tab:main}.}
\label{fig:heatmap-persona}
\end{figure}

\subsection{Human Feedback and Evaluation}\label{app:human}
One author-annotator supplied \textsc{Image Captioning} feedback under the \textit{Picture Book} and \textit{Magazine Editor} descriptions and evaluated the resulting models on separate test images.
These are instructed preferences from one annotator.
The feedback contained $100$ and $89$ decided comparisons, respectively, sampled from the merge-family pool.
The pairs had valid decided LLM judgments, but only the human choices and reward scores were used to fit one posterior per description.

The annotator evaluated four independently sampled packs of $100$ test images, comparing each model with uniform and comparing the two personalized models under each description.
Model identities were concealed and presentation sides randomized; the annotator knew the research hypothesis.
Byte-identical captions count as ties under a rule fixed before evaluation.
Table~\ref{tab:human} shows that both models outperform uniform and that the preferred model changes with the instructed persona.
This single-annotator study demonstrates feasibility under instructed preferences, rather than independent validation across human users.

\begin{table}[ht]
\centering\small
\caption{\textbf{Human evaluation on Image Captioning.} The win rate refers to the model fitted to the named preference description. Each pack has $100$ images. Intervals are $95\%$ Wilson intervals on decided comparisons.}
\label{tab:human}
\begin{tabular}{llrrrr}
\toprule
Description & Comparator & Wins & Losses & Ties & Win rate (\%)\\
\midrule
Picture Book & Uniform & 44 & 8 & 48 & 84.6 [72.5, 92.0]\\
Magazine Editor & Uniform & 84 & 7 & 9 & 92.3 [85.0, 96.2]\\
Picture Book & Magazine Editor model & 89 & 5 & 6 & 94.7 [88.1, 97.7]\\
Magazine Editor & Picture Book model & 93 & 3 & 4 & 96.9 [91.2, 98.9]\\
\bottomrule
\end{tabular}
\end{table}

\subsection{Length- and Readability-Matched Preference Reversals}\label{app:length}
To examine whether the reversals extend beyond length and readability, we restrict the six persona-pair comparisons to outputs whose word counts differ by at most $10\%$ of their mean length.
The joint match also requires their Flesch--Kincaid grades to differ by at most $1$.
This analysis uses $n=500$ and the first feedback subset, with $2{,}000$ paired input-bootstrap samples.

For \textit{Librarian}--\textit{Literary}, the \textit{Librarian} model wins $54/55$ length-matched comparisons under the \textit{Librarian} judge and $5/49$ under the Literary judge.
With readability also matched, the counts are $33/33$ and $3/34$.
The corresponding between-judge differences are $88$ percentage points with $95\%$ interval $[78,96]$ and $91$ points with interval $[81,100]$.
Matching selects a subset of generated outputs and does not identify a causal effect with length removed.

\subsection{Rubric-Based Preference Following}\label{app:rubric-results}
Table~\ref{tab:rubric-full} reports persona-level results for the rubric definitions in Appendix~\ref{app:forma}.
The macro advantage holds for every compared method at both.

For the sweep in \Cref{fig:analysis}b, we vary the feedback rubric as
$(1-t)\mathbf h_{\mathrm{Literary}}+t\mathbf h_{\mathrm{Genre}}$
for $t\in\{0,0.17,0.33,0.50,0.67,0.83,1\}$ and fit a posterior at each setting using $n=500$ comparisons from each of three feedback subsets.
Both endpoint judges remain fixed during evaluation against uniform.
These rubric weights are hidden from \bpm and need not equal the inferred weights over the expert library's reward dimensions.

\begin{table}[ht]
\centering\footnotesize
\caption{\textbf{Rubric-persona decided win rates.} BPM's win rates (\%) against each comparator. Entries average three feedback subsets and include $95\%$ paired input-bootstrap intervals.}
\label{tab:rubric-full}
\begin{tabular}{llcccc}
\toprule
Persona & $n$ & Uniform & Base Policy & Ridge-Merge & BT-BoN\\
\midrule
\textit{Literary} & 100 & 69.4 [65.9, 72.9] & 80.6 [77.7, 83.5] & 61.6 [58.8, 64.4] & 67.4 [64.5, 70.7]\\
 & 500 & 82.6 [79.7, 85.2] & 90.2 [88.2, 92.2] & 77.6 [74.9, 80.1] & 81.9 [79.0, 84.7]\\
\textit{Genre} & 100 & 65.7 [62.3, 69.0] & 68.0 [64.5, 71.2] & 60.8 [58.1, 63.5] & 64.0 [60.8, 67.1]\\
 & 500 & 68.8 [65.4, 72.1] & 70.8 [67.6, 74.2] & 66.1 [63.6, 68.5] & 65.9 [62.7, 69.0]\\
\textit{Librarian} & 100 & 59.6 [55.8, 63.2] & 49.8 [45.7, 53.5] & 52.3 [49.4, 55.1] & 64.4 [61.2, 67.5]\\
 & 500 & 67.2 [63.7, 70.6] & 57.7 [53.9, 61.5] & 61.1 [58.4, 63.9] & 68.9 [65.6, 72.0]\\
\textit{Flash} & 100 & 66.4 [62.9, 69.8] & 74.5 [71.3, 77.6] & 66.1 [63.8, 68.5] & 55.1 [51.7, 58.3]\\
 & 500 & 72.7 [69.4, 75.7] & 79.1 [76.2, 81.9] & 75.8 [73.5, 78.1] & 59.5 [56.0, 62.8]\\
\textit{Screenwriter} & 100 & 67.6 [63.9, 71.0] & 81.4 [78.3, 84.1] & 61.4 [58.4, 64.1] & 66.1 [62.7, 69.3]\\
 & 500 & 82.4 [79.3, 85.2] & 91.0 [88.9, 92.9] & 76.6 [74.0, 79.0] & 79.1 [75.6, 82.2]\\
\bottomrule
\end{tabular}
\end{table}

\subsection{Using Inferred Weights for Policy Training}\label{app:retrain}
We use \textit{Ward}'s and \textit{ED}'s posterior-mean weights from the first feedback subset at $n=500$ to scalarize the $29$ rewards for GRPO training on $1{,}600$ Summarization prompts.
A third policy uses uniform reward weights under the same training recipe.
Each persona prefers the policy trained with its inferred weights over both uniform-weight training and the other persona's policy (Table~\ref{tab:retrain}).

\begin{table}[ht]
\centering\small
\caption{Retrained policies at inferred weights: decided win rates (\%) under the named persona's judge.
The last column compares each persona's retrained policy with the other persona's retrained policy.}
\label{tab:retrain}
\begin{tabular}{lccc}
\toprule
Judge & vs.\ uniform-weight retrain & vs.\ \bpm merge & vs.\ other persona\\
\midrule
Ward & 98.37 [97.20,99.53] & 59.09 [36.36,77.27] & 100.00 [99.06,100.00]\\
ED & 85.40 [81.37,89.13] & 48.76 [41.79,55.72] & 97.95 [96.59,99.09]\\
\bottomrule
\end{tabular}
\end{table}

\subsection{Alternative Expert Libraries}\label{app:library-robustness}
The RLOO and DPO libraries each contain one Summarization expert per reward, trained without persona feedback.
We retain the feedback protocol and persona descriptions and compare each personalized model with its library's own uniform merge.
Table~\ref{tab:library-personas} shows gains for every persona at both budgets, averaged over three feedback subsets.
Since each library has a different uniform opponent, these results assess personalization within libraries rather than rank their training algorithms.

\begin{table}[tbp]
\centering\footnotesize
\setlength{\tabcolsep}{4pt}
\caption{\textbf{Personalization within alternative expert libraries.} \bpm's decided win rates (\%) against each library's own uniform merge, averaged over three feedback subsets, with $95\%$ input-bootstrap intervals. Each comparison uses its library's opponent.}
\label{tab:library-personas}
\begin{tabular}{lcccc}
\toprule
 & \multicolumn{2}{c}{RLOO experts} & \multicolumn{2}{c}{DPO experts}\\
Persona & $n=100$ & $n=500$ & $n=100$ & $n=500$ \\
\midrule
Ward & 92.4 [89.7, 94.9] & 97.4 [96.0, 98.5] & 74.9 [68.7, 80.6] & 87.2 [83.8, 90.2] \\
ED & 94.8 [93.0, 96.3] & 96.4 [94.8, 97.8] & 86.8 [83.6, 89.9] & 93.4 [90.9, 95.5] \\
Complete & 93.3 [90.8, 95.6] & 99.2 [98.5, 99.8] & 78.5 [72.7, 83.8] & 92.3 [89.5, 94.8] \\
Concise & 96.5 [95.2, 97.7] & 97.9 [96.4, 99.1] & 83.2 [77.8, 87.9] & 92.6 [90.1, 94.7] \\
Faithful & 94.3 [91.9, 96.5] & 98.2 [97.0, 99.3] & 72.7 [65.9, 79.4] & 89.2 [85.4, 92.6] \\
Calibrated & 92.2 [89.5, 94.8] & 97.1 [95.5, 98.5] & 70.3 [64.4, 75.8] & 87.1 [82.5, 91.1] \\
\bottomrule
\end{tabular}
\end{table}

\section{Additional Analysis}\label{app:analysis}

\subsection{Posterior Distributions and Uncertainty}\label{app:uncertainty-performance}
\paragraph{Persona-specific posterior distributions.}
To illustrate what \bpm provides beyond a single coefficient vector, we compare the shared prior with the reward-weight posteriors for the \textit{Literary} and \textit{Librarian} personas on Story Generation at $n=500$.
\Cref{fig:posterior-story} shows six selected marginal distributions.
The Literary posterior assigns greater weight to show-don't-tell, prose elegance, and imagery, whereas the Librarian posterior emphasizes avoiding purple prose and coherence.
The means supply coefficients for constructing the personalized models; the distributions additionally quantify support for alternative values of each weight under the model and prior.
These examples illustrate persona-specific updates and the information retained by the posterior beyond the coefficients used for merging.

\begin{figure}[t]
\centering
\includegraphics[width=\textwidth]{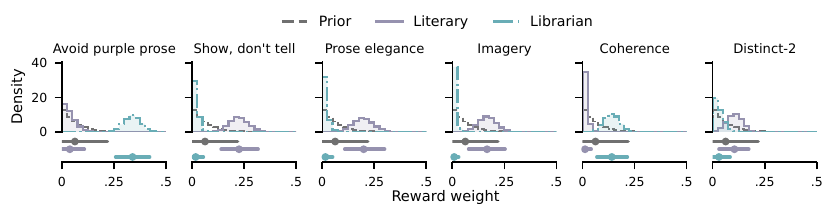}
\caption{\textbf{Persona-specific reward-weight distributions.}
The shared prior and the Literary and Librarian posteriors on Story Generation at $n=500$, for six selected reward dimensions.
Dots and horizontal bars denote means and central $95\%$ intervals.
These are marginal distributions of $w$, rather than intervals for the temperature-scaled coefficients $\beta=w/\tau$.}
\label{fig:posterior-story}
\end{figure}

\paragraph{Uncertainty and personalization performance.}
For each persona and feedback budget, we average posterior uncertainty $U$ and the decided win rate against uniform merging separately over the three feedback subsets.
We compute Spearman correlations between these averages within each task and budget, using six Summarization personas, six Image Captioning personas, and seven Story Generation personas.
The six correlations in \Cref{fig:uncertainty}b therefore describe between-persona variation with task and budget held fixed.
Lower uncertainty is associated with higher preference win rates
in some task--budget setting, providing context for interpreting posterior uncertainty.

\subsection{Coefficient Interval Coverage}\label{app:coverage}
\paragraph{Calibration under the matched model and prior.}
We compare interval procedures on shared simulated feedback with a common target $\boldsymbol\beta^\star=\mathbf w^\star/\tau^\star$.
Generating weights and temperatures follow \bpm's prior, and choices follow the Bradley--Terry likelihood on each task's normalized contrasts.
Coverage and width are averaged over coefficient coordinates and simulations within each task, then equally over tasks.
Table~\ref{tab:def-intervals} characterizes the tested procedures under this generating distribution; it does not establish calibration for observed persona preferences.

\paragraph{The prior contributes substantially at small budgets.}
Prior-only intervals attain near-nominal coverage and are slightly narrower on average than the posterior intervals at $n=10$ and $25$.
Thus, the matched comparison does not demonstrate interval contraction due to feedback.
This does not imply that the posterior equals the prior: posterior location and shape can change without a decrease in average interval width.
The $\mathbf w$ marginals at $n=500$ in Appendix~\ref{app:uncertainty-performance} illustrate persona-specific updates in a separate setting, while the stress tests assess sensitivity to generating and fitted assumptions.

\paragraph{Comparison with regularized estimation.}
We next examine whether regularized point estimation combined with standard interval constructions reproduces this behavior.
We fit RM's logistic estimator to the same normalized contrasts and construct bootstrap and Laplace intervals for its raw coefficient $\beta$.
This coefficient is unconstrained and regularized toward zero, whereas \bpm constrains $\beta=w/\tau$ through simplex weights.
RM bootstrap intervals under-cover despite being wider than \bpm's; Laplace intervals instead over-cover and are much wider (\Cref{tab:def-intervals}).
Thus, the two tested constructions do not reproduce \bpm's interval performance in this simulation, illustrating that regularizing a point estimate alone does not determine the quality of its uncertainty estimates.

\begin{table}[t]
\centering
\small
\caption{\textbf{Coefficient intervals under limited feedback.}
Coverage (\%) and mean width of nominal $90\%$ intervals for
$\beta=w/\tau$, averaged equally over three tasks after averaging
over coefficient coordinates and simulations within each task.
The generating parameters follow \bpm's prior.
MAP, MLE, and RM bootstrap intervals use 200 resamples.
The SS row uses sequential Gaussian inference, distinct from its
signed point estimator used for merging.
Prior-only intervals do not condition on feedback.}
\label{tab:def-intervals}
\setlength{\tabcolsep}{7pt}
\begin{tabular}{lrrrr}
\toprule
 & \multicolumn{2}{c}{$n=10$} & \multicolumn{2}{c}{$n=25$} \\
\cmidrule(lr){2-3}\cmidrule(lr){4-5}
Interval & Coverage $\uparrow$ & Width $\downarrow$ & Coverage $\uparrow$ & Width $\downarrow$ \\
\midrule
\bpm posterior & 88.9 & 0.176 & 89.2 & 0.191 \\
Prior only & 89.8 & 0.171 & 89.9 & 0.171 \\
MAP bootstrap & 8.2 & 0.024 & 20.4 & 0.261 \\
MLE bootstrap & 53.1 & 11.144 & 61.6 & 8.974 \\
SS sequential Gaussian & 99.9 & 3.036 & 99.7 & 2.797 \\
RM bootstrap & 72.8 & 0.599 & 80.0 & 0.913 \\
RM Laplace & 99.9 & 3.052 & 99.7 & 2.812 \\
\bottomrule
\end{tabular}
\end{table}

\paragraph{Sensitivity to modeling assumptions.}
To examine the scope of these results, we vary the generating weights, the fitted prior, and the amount of label noise.
\Cref{tab:def-stress} summarizes 16 settings across three tasks at $n\in\{10,25,100\}$, with 50 replicates per setting, task, and budget.
Coverage of the simplex weights $w$ remains near nominal under the tested label-flip rates, but falls to $70.6$--$77.1\%$ when the fitted Dirichlet concentration is doubled.
For generating weights concentrated on three dimensions, coverage on Story Generation falls to $81.1$--$81.4\%$.
Sensitivity also depends on the coefficient target: halving the standard deviation of the $\log\tau$ prior yields $87.7$--$90.4\%$ coverage for $w$, but $78.6$--$88.6\%$ for $\beta$.
Across the tested settings, mean posterior interval widths for $w$ are approximately $92$--$100\%$ of the corresponding prior widths.
These findings identify both the contribution of the prior and the limits of the resulting uncertainty estimates: \bpm characterizes plausible reward weights conditional on its modeling assumptions, and limited feedback need not overcome prior misspecification.
For sparse generating weights, coordinate-averaged coverage can also conceal missed intervals for the few large weights.
These simulations assess parameter-interval sensitivity.
They do not establish that personalized-model performance is robust to changing the fitted prior in the final expert libraries.

\begin{table}[t]
\centering
\small
\caption{\textbf{Sensitivity of interval coverage to modeling assumptions.}
Each entry is the range of mean nominal $90\%$ coverage (\%) over the nine task--budget cells in a setting, not a confidence interval.
Coverage is averaged over coordinates and simulation replicates within each cell. Separate columns distinguish simplex weights $w$ from temperature-scaled coefficients $\beta=w/\tau$.
Prior-sensitivity and label-noise settings draw generating parameters from the production prior; multipliers modify the fitted prior.}
\label{tab:def-stress}
\label{tab:def-stress-beta}
\setlength{\tabcolsep}{8pt}
\begin{tabular}{lcc}
\toprule
Setting & $w$ coverage range & $\beta$ coverage range \\
\midrule
\multicolumn{3}{l}{\emph{Generating weight configurations}} \\
One weight $0.8$, $\tau^*=0.5$ & 93.6--96.6 & 93.8--96.9 \\
One weight $0.8$, $\tau^*=1$ & 93.8--96.6 & 93.8--96.8 \\
One weight $0.8$, $\tau^*=2$ & 93.8--96.6 & 93.8--96.6 \\
Three weights $0.3$, $\tau^*=0.5$ & 81.2--90.8 & 81.2--91.0 \\
Three weights $0.3$, $\tau^*=1$ & 81.1--89.9 & 82.3--90.8 \\
Three weights $0.3$, $\tau^*=2$ & 81.2--89.7 & 91.2--98.5 \\
Near-uniform, $\tau^*=0.5$ & 100.0 & 84.0--99.7 \\
Near-uniform, $\tau^*=1$ & 100.0 & 100.0 \\
Near-uniform, $\tau^*=2$ & 100.0 & 100.0 \\
\midrule
\multicolumn{3}{l}{\emph{Fitted prior specification}} \\
Dirichlet concentration $\times0.5$ & 92.4--98.5 & 91.5--96.9 \\
Dirichlet concentration $\times1$ & 87.2--91.0 & 88.0--91.6 \\
Dirichlet concentration $\times2$ & 70.6--77.1 & 71.6--83.9 \\
$\log\tau$ prior sd $\times0.5$ & 87.7--90.4 & 78.6--88.6 \\
$\log\tau$ prior sd $\times2$ & 86.6--91.2 & 87.5--97.1 \\
\midrule
\multicolumn{3}{l}{\emph{Feedback noise}} \\
$5\%$ label flips & 88.3--91.0 & 87.9--91.7 \\
$15\%$ label flips & 87.2--91.0 & 85.6--89.9 \\
\bottomrule
\end{tabular}
\end{table}

\paragraph{Inference and simulation details.}
The matched comparison uses 100 simulations per task and budget, with $w\sim\operatorname{Dir}(s\mathbf{1})$ and $\log\tau\sim\mathcal{N}(\log 2,1)$.
\bpm intervals are central posterior intervals; MAP and MLE intervals are percentile intervals from 200 bootstrap refits.
Failed point-estimate fits are retried and then discarded; simulations without surviving refits count as noncoverage.
SS uses the sequential Gaussian update of Chow et al.\ (2026), initialized with $\mathcal{N}(0,I)$ and without the signed point estimator's $\ell_1$ constraint; marginal intervals are the Gaussian mean plus or minus $1.645$ marginal standard
deviations. The iterated-update variant gives similar results.
Prior intervals use $400{,}000$ Monte Carlo draws.
RM uses no intercept and a fixed $\ell_2$ penalty of one, with the
same 200 bootstrap resamples as MAP and MLE.
Its Laplace intervals are $\hat\beta_j\pm1.645\sqrt{(H^{-1})_{jj}}$, where $H$ is the Hessian of the penalized objective.
For this comparison, RM is refitted on the shared $\ell_\infty$-normalized contrasts, rather than the pre-normalization standardized contrasts used in its model-personalization pipeline; its intervals concern raw logistic coefficients, not softmax weights.

In the stress tests, the mass remaining after assigning one weight of $0.8$ or three weights of $0.3$ is shared equally among the other dimensions; near-uniform weights are drawn from $\operatorname{Dir}(50\mathbf{1})$.
Prior-sensitivity settings change only the stated fitted-prior parameter, and label-noise settings independently flip each simulated choice with the stated probability.

\subsection{Point Estimates and Coefficient Stability}\label{app:point-estimates}
\paragraph{Stability under limited feedback.}
We measure coefficient variability by the mean pairwise $\ell_1$ distance between fits to three feedback subsets, averaged over personas.
Tables~\ref{tab:stability-differences-1}--\ref{tab:stability-differences-3} report paired differences from \bpm.
At $n=25$ and $100$, posterior means vary less than MLE in all three tasks.
Comparisons with regularized estimators depend on the task and budget: RM has lower variability on Summarization at $n=100$ and Story Generation at $n=100,500$, while MAP has lower variability on Story Generation at $n=500$.
Adapted SS differences on Image Captioning remain unresolved beyond $n=10$.
These results support reduced sensitivity relative to MLE under limited feedback, without establishing uniformly minimal coefficient variation.

\paragraph{Coefficient stability and output preference.}
MLE and adapted SS often yield more preferred outputs than \bpm, whereas \bpm outperforms RM at both main-table budgets.
The posterior mean therefore provides a stable coefficient estimate in the stated settings alongside posterior uncertainty, rather than uniformly maximizing output preference.
Coefficient stability does not establish stability of generated behavior: a partial-score diagnostic found higher coefficient variation but slightly lower output-score variation for shrinkage on Summarization at $n=100$ (Supplementary Material, Section S1).

\paragraph{A shrinkage control.}
We interpolate between MLE and uniform weights,
$\widehat{\mathbf w}_{\alpha}=(1-\alpha)\widehat{\mathbf w}_{\mathrm{MLE}}+\alpha\mathbf1/L$.
Table~\ref{tab:def-shrinkage} shows task-dependent trade-offs: shrinkage yields more preferred outputs but less stable coefficients on Summarization, while \bpm has higher output preference on Image Captioning and Story Generation at $n=100$.
At $n=25$ on the latter two tasks, shrinkage selects uniform weights.
This control does not consistently reproduce the posterior mean's combination of coefficient stability and output preference.

\paragraph{Estimation and statistical details.}
The posterior mean, MAP, and MLE share the preference likelihood; MLE omits the weight and temperature priors.
Because the chosen Dirichlet density is unbounded toward the simplex boundary, MAP is defined as the joint mode in centered softmax coordinates and $\log\tau$, including the transformation Jacobian, and optimized using L-BFGS-B with uniform initialization and eight seeded restarts.
RM variability uses its softmax-transformed merge coefficients; adapted SS uses signed coefficients with $\ell_1$ norm at most one.
Feedback repetitions are separately sampled and nested across budgets within each repetition.
Variability intervals use 2,000 paired persona-bootstrap resamples, holding the three feedback subsets fixed.
The extended comparison is post hoc.
Shrinkage selects $\alpha\in\{0,0.05,\ldots,1\}$ by persona-averaged held-out negative log-likelihood, averaging over feedback draws and refitting temperature on each candidate's training rows.
Each coefficient fit uses $n$ comparisons; selection additionally uses 100 held-out elicitation comparisons per split and delivers one $\alpha$ per task and budget.

\begin{table}[t]
\centering\small
\caption{\textbf{Shrinkage toward uniform compared with the BPM posterior mean.}
Win rates are shrinkage against BPM, with 95\% input-bootstrap intervals.
Variability is the persona-averaged pairwise coefficient $\ell_1$ distance
across three feedback subsets. Variability entries are point estimates;
they do not imply equivalent stability when close.
Shrinkage selection additionally uses held-out elicitation comparisons.}
\label{tab:def-shrinkage}
\begin{tabular}{lrcccc}
\toprule
Task & $n$ & $\alpha^*$ & Shrinkage win rate (\%) & Shrinkage variability & BPM variability \\
\midrule
Summarization & 25 & 0.70 & 77.6 [75.6, 79.4] & 0.299 & 0.112 \\
Summarization & 100 & 0.20 & 69.8 [67.8, 71.7] & 0.646 & 0.465 \\
Image Captioning & 100 & 0.25 & 41.9 [40.0, 43.8] & 0.733 & 0.313 \\
Story Generation & 100 & 0.60 & 44.5 [43.3, 45.7] & 0.328 & 0.392 \\
\bottomrule
\end{tabular}
\end{table}

\begin{table}[t]\centering\small
\caption{\textbf{Coefficient variability on Summarization.} Variability intervals use 2,000 persona-bootstrap resamples. Difference intervals are paired (method minus BPM), holding the three feedback subsets fixed. Intervals containing zero do not demonstrate equal stability. SS coefficients are signed.}
\label{tab:stability-differences-1}
\begin{tabular}{llcc}\toprule
$n$ & Method & Variability [95\% CI] & Difference from BPM [95\% CI] \\ \midrule
25 & BPM & 0.112 [0.080, 0.142] & -- \\
25 & MAP & 0.231 [0.098, 0.377] & +0.119 [+0.014, +0.240] \\
25 & MLE & 0.998 [0.494, 1.482] & +0.886 [+0.412, +1.346] \\
25 & Shrinkage & 0.299 [0.148, 0.445] & +0.187 [+0.067, +0.307] \\
25 & RM & 0.331 [0.297, 0.367] & +0.219 [+0.173, +0.263] \\
25 & SS & 1.498 [1.216, 1.705] & +1.386 [+1.092, +1.587] \\
25 & ES & 1.556 [0.889, 2.000] & +1.443 [+0.807, +1.892] \\
25 & Uniform & 0.000 [0.000, 0.000] & -0.112 [-0.142, -0.080] \\
100 & BPM & 0.465 [0.364, 0.561] & -- \\
100 & MAP & 0.568 [0.475, 0.657] & +0.103 [-0.045, +0.265] \\
100 & MLE & 0.807 [0.671, 1.000] & +0.342 [+0.245, +0.457] \\
100 & Shrinkage & 0.646 [0.537, 0.800] & +0.180 [+0.105, +0.261] \\
100 & RM & 0.353 [0.345, 0.361] & -0.113 [-0.204, -0.020] \\
100 & SS & 0.935 [0.719, 1.107] & +0.470 [+0.243, +0.698] \\
100 & ES & 1.222 [0.667, 1.667] & +0.757 [+0.194, +1.248] \\
100 & Uniform & 0.000 [0.000, 0.000] & -0.465 [-0.561, -0.364] \\
500 & BPM & 0.238 [0.186, 0.298] & -- \\
500 & MAP & 0.234 [0.191, 0.287] & -0.004 [-0.052, +0.049] \\
500 & MLE & 0.277 [0.206, 0.355] & +0.039 [-0.042, +0.131] \\
500 & Shrinkage & 0.277 [0.206, 0.355] & +0.039 [-0.042, +0.131] \\
500 & RM & 0.206 [0.151, 0.272] & -0.032 [-0.136, +0.084] \\
500 & SS & 0.205 [0.158, 0.238] & -0.033 [-0.105, +0.037] \\
500 & ES & 0.222 [0.000, 0.667] & -0.016 [-0.296, +0.462] \\
500 & Uniform & 0.000 [0.000, 0.000] & -0.238 [-0.298, -0.186] \\
\bottomrule\end{tabular}\end{table}

\begin{table}[t]\centering\small
\caption{\textbf{Coefficient variability on Image Captioning.} Variability intervals use 2,000 persona-bootstrap resamples. Difference intervals are paired (method minus BPM), holding the three feedback subsets fixed. Intervals containing zero do not demonstrate equal stability. SS coefficients are signed.}
\label{tab:stability-differences-2}
\begin{tabular}{llcc}\toprule
$n$ & Method & Variability [95\% CI] & Difference from BPM [95\% CI] \\ \midrule
25 & BPM & 0.358 [0.214, 0.495] & -- \\
25 & MAP & 0.366 [0.205, 0.540] & +0.007 [-0.209, +0.245] \\
25 & MLE & 1.592 [1.433, 1.747] & +1.234 [+0.938, +1.511] \\
25 & Shrinkage & 0.000 [0.000, 0.000] & -0.358 [-0.495, -0.214] \\
25 & RM & 0.323 [0.261, 0.389] & -0.036 [-0.212, +0.157] \\
25 & SS & 0.742 [0.174, 1.375] & +0.383 [-0.277, +1.138] \\
25 & ES & 0.889 [0.222, 1.667] & +0.531 [-0.260, +1.419] \\
25 & Uniform & 0.000 [0.000, 0.000] & -0.358 [-0.495, -0.214] \\
100 & BPM & 0.313 [0.203, 0.463] & -- \\
100 & MAP & 0.514 [0.335, 0.721] & +0.200 [+0.117, +0.312] \\
100 & MLE & 0.978 [0.594, 1.320] & +0.664 [+0.340, +0.915] \\
100 & Shrinkage & 0.733 [0.446, 0.990] & +0.420 [+0.179, +0.597] \\
100 & RM & 0.366 [0.324, 0.417] & +0.053 [-0.049, +0.143] \\
100 & SS & 0.563 [0.076, 1.055] & +0.250 [-0.171, +0.748] \\
100 & ES & 0.556 [0.000, 1.222] & +0.242 [-0.352, +0.963] \\
100 & Uniform & 0.000 [0.000, 0.000] & -0.313 [-0.463, -0.203] \\
500 & BPM & 0.219 [0.145, 0.290] & -- \\
500 & MAP & 0.241 [0.179, 0.315] & +0.022 [-0.023, +0.062] \\
500 & MLE & 0.332 [0.240, 0.423] & +0.113 [+0.009, +0.225] \\
500 & RM & 0.199 [0.160, 0.251] & -0.020 [-0.122, +0.095] \\
500 & SS & 0.155 [0.035, 0.320] & -0.065 [-0.170, +0.067] \\
500 & ES & 0.444 [0.000, 0.889] & +0.225 [-0.220, +0.697] \\
500 & Uniform & 0.000 [0.000, 0.000] & -0.219 [-0.290, -0.145] \\
\bottomrule\end{tabular}\end{table}

\begin{table}[t]\centering\small
\caption{\textbf{Coefficient variability on Story Generation.} Variability intervals use 2,000 persona-bootstrap resamples. Difference intervals are paired (method minus BPM), holding the three feedback subsets fixed. Intervals containing zero do not demonstrate equal stability. SS coefficients are signed.}
\label{tab:stability-differences-3}
\begin{tabular}{llcc}\toprule
$n$ & Method & Variability [95\% CI] & Difference from BPM [95\% CI] \\ \midrule
25 & BPM & 0.088 [0.059, 0.116] & -- \\
25 & MAP & 0.119 [0.056, 0.190] & +0.032 [-0.010, +0.079] \\
25 & MLE & 1.398 [1.249, 1.554] & +1.310 [+1.152, +1.482] \\
25 & Shrinkage & 0.000 [0.000, 0.000] & -0.088 [-0.116, -0.059] \\
25 & RM & 0.439 [0.389, 0.493] & +0.351 [+0.297, +0.414] \\
25 & SS & 1.306 [1.113, 1.501] & +1.219 [+1.038, +1.405] \\
25 & Uniform & 0.000 [0.000, 0.000] & -0.088 [-0.116, -0.059] \\
100 & BPM & 0.392 [0.333, 0.447] & -- \\
100 & MAP & 0.413 [0.357, 0.469] & +0.021 [+0.014, +0.029] \\
100 & MLE & 0.819 [0.720, 0.915] & +0.428 [+0.333, +0.525] \\
100 & Shrinkage & 0.328 [0.288, 0.366] & -0.064 [-0.109, +0.001] \\
100 & RM & 0.260 [0.221, 0.298] & -0.132 [-0.163, -0.098] \\
100 & SS & 0.892 [0.809, 0.963] & +0.500 [+0.421, +0.574] \\
100 & Uniform & 0.000 [0.000, 0.000] & -0.392 [-0.447, -0.333] \\
500 & BPM & 0.130 [0.097, 0.160] & -- \\
500 & MAP & 0.122 [0.091, 0.148] & -0.008 [-0.014, -0.003] \\
500 & MLE & 0.166 [0.127, 0.199] & +0.036 [+0.022, +0.053] \\
500 & RM & 0.036 [0.026, 0.045] & -0.094 [-0.119, -0.069] \\
500 & SS & 0.175 [0.114, 0.228] & +0.045 [+0.014, +0.077] \\
500 & Uniform & 0.000 [0.000, 0.000] & -0.130 [-0.160, -0.097] \\
\bottomrule\end{tabular}\end{table}

\paragraph{Adapted SS and merge-operator controls.}
\label{app:ss}

We compare \bpm with an adapted Spectral Souping (SS) coefficient estimator~\citep{chow2026spectral} to examine how their inferred coefficients affect generated outputs.
SS uses the same fixed offline feedback, with standardized reward-score differences replacing its policy-derived features.
It fits signed coefficients $\boldsymbol\lambda\in\mathbb R^L$ using an unregularized Bradley--Terry objective subject to $\|\boldsymbol\lambda\|_1\le\beta'/\beta$.
We set $\beta'=\beta$ without tuning.
This experiment evaluates the adapted coefficient estimator; Appendix~\ref{app:coverage} separately examines coefficient intervals, including those from SS's sequential Gaussian inference.

\paragraph{Signed coefficients require a different merge operator.}
\bpm's factor-averaged LoRA merge weights both adapter factors, so assigning $-1$ to a single expert produces $+\Delta\boldsymbol\theta_k$ rather than subtracting that expert (Appendix~\ref{app:merge}).
We therefore implement signed SS coefficients with exact delta merging,
$\boldsymbol\theta_0+\sum_k\lambda_k\Delta\boldsymbol\theta_k$.
We report the output comparison here, together with a control applying this same operator to \bpm's coefficients.

\paragraph{The adapted SS models often receive higher preference.}
Table~\ref{tab:ss} shows that the SS models are preferred in $14$ of $15$ task--budget comparisons; the remaining interval, for \textsc{Summarization} at $n=500$, includes $50\%$.
Applying exact delta merging to both estimators preserves the qualitative conclusions at $n=100$ and $500$ (Table~\ref{tab:ss-operator}).
Changing the merge operator alone therefore does not close the performance gap.
\begin{table}[htbp]
\centering\small
\setlength{\tabcolsep}{4pt}
\caption{\textbf{Comparison with the adapted Spectral Souping estimator.}
\bpm's macro decided win rate (\%), averaged equally over personas and over three feedback subsets, with $95\%$ paired input-bootstrap intervals.
Values below $50\%$ favor the adapted SS models.
\bpm uses the factor-averaged merge from the main experiments; SS uses the exact delta merge described above.}
\label{tab:ss}
\begin{tabular}{lccccc}
\toprule
Task & $n{=}10$ & $n{=}25$ & $n{=}50$ & $n{=}100$ & $n{=}500$ \\
\midrule
\textsc{Summarization} & $32.1$ & $17.5$ & $22.9$ & $24.3$ & $50.2$ \\
 & {\scriptsize$[30.7, 33.5]$} & {\scriptsize$[16.0, 19.0]$} & {\scriptsize$[20.4, 25.1]$} & {\scriptsize$[21.7, 26.9]$} & {\scriptsize$[44.6, 55.9]$} \\
\textsc{Image Captioning} & $15.8$ & $15.9$ & $14.3$ & $15.6$ & $19.2$ \\
 & {\scriptsize$[14.7, 16.9]$} & {\scriptsize$[14.7, 17.1]$} & {\scriptsize$[13.1, 15.5]$} & {\scriptsize$[14.5, 16.8]$} & {\scriptsize$[17.7, 20.7]$} \\
\textsc{Story Generation} & $38.6$ & $34.1$ & $33.4$ & $33.2$ & $33.7$ \\
 & {\scriptsize$[37.2, 39.9]$} & {\scriptsize$[32.9, 35.4]$} & {\scriptsize$[32.2, 34.6]$} & {\scriptsize$[32.0, 34.5]$} & {\scriptsize$[32.3, 35.0]$} \\
\bottomrule
\end{tabular}
\end{table}

\begin{table}[htbp]
\centering\small
\setlength{\tabcolsep}{4pt}
\caption{\textbf{Controlling the merge operator.}
Macro decided win rates (\%) with $95\%$ paired input-bootstrap intervals, averaged over three feedback subsets.
Each column reports direct preference for the first model over the second.
Exact denotes a linear combination of adapter deltas; factor denotes \bpm's main factor-averaged merge.}
\label{tab:ss-operator}
\begin{tabular}{lccc}
\toprule
& \multicolumn{2}{c}{BPM exact vs.\ SS exact} & BPM exact vs.\ BPM factor \\
\cmidrule(lr){2-3}\cmidrule(lr){4-4}
Task & $n=100$ & $n=500$ & $n=500$ \\
\midrule
Summarization & $27.1\ [24.4,29.8]$ & $45.5\ [39.4,52.4]$ & $49.8\ [41.6,58.6]$ \\
Image Captioning & $16.9\ [15.6,18.2]$ & $20.7\ [19.1,22.1]$ & $57.5\ [55.2,59.8]$ \\
Story Generation & $34.7\ [33.3,36.0]$ & $36.4\ [35.1,37.7]$ & $53.0\ [51.7,54.1]$ \\
\bottomrule
\end{tabular}
\end{table}

\paragraph{Higher preference does not imply more stable coefficients.}
We measure sensitivity to feedback using the mean pairwise $\ell_1$ distance between coefficients fitted to three subsets, averaged equally over personas.
\bpm is more stable at $n=10$ in all three tasks (Table~\ref{tab:ss-stability}).
Paired intervals support lower variation through $n=250$ on \textsc{Summarization} and at every tested budget on \textsc{Story Generation}; on \textsc{Image Captioning}, the difference is inconclusive beyond $n=10$.
These results support \bpm's stability under limited feedback without implying uniformly better output preference.

\begin{table}[htbp]
\centering\small
\setlength{\tabcolsep}{4pt}
\caption{\textbf{Coefficient variation across feedback subsets.}
Persona-averaged pairwise $\ell_1$ distance; lower values indicate greater stability.
Within each task, bold marks the lower value only when the $95\%$ paired persona-bootstrap interval for SS minus BPM excludes zero.
BPM uses simplex weights; SS uses signed coefficients with $\ell_1$ norm at most one.}
\label{tab:ss-stability}
\begin{tabular}{lcccccc}
\toprule
& \multicolumn{2}{c}{Summarization} & \multicolumn{2}{c}{Image Captioning} & \multicolumn{2}{c}{Story Generation} \\
\cmidrule(lr){2-3}\cmidrule(lr){4-5}\cmidrule(lr){6-7}
$n$ & BPM & SS & BPM & SS & BPM & SS \\
\midrule
10 & $\mathbf{0.071}$ & $1.827$ & $\mathbf{0.102}$ & $1.014$ & $\mathbf{0.035}$ & $1.661$ \\
25 & $\mathbf{0.112}$ & $1.498$ & $0.358$ & $0.742$ & $\mathbf{0.088}$ & $1.306$ \\
50 & $\mathbf{0.304}$ & $1.383$ & $0.257$ & $0.660$ & $\mathbf{0.278}$ & $1.132$ \\
100 & $\mathbf{0.465}$ & $0.935$ & $0.313$ & $0.563$ & $\mathbf{0.392}$ & $0.892$ \\
250 & $\mathbf{0.398}$ & $0.462$ & $0.259$ & $0.277$ & $\mathbf{0.295}$ & $0.425$ \\
500 & $0.238$ & $0.205$ & $0.219$ & $0.155$ & $\mathbf{0.130}$ & $0.175$ \\
\bottomrule
\end{tabular}
\end{table}

\subsection{Direct Weighting}\label{app:dilution}
Given inferred reward weights, \bpm uses them directly as coefficients for merging the shared experts.
An alternative is to use the same weights to score candidate models and select the one with the highest inferred utility.
We examine whether this alternative improves the fitted reward objective and whether it produces outputs that better match persona preferences.

\paragraph{Does direct merging maximize inferred utility?}
For each fitted $\widehat{\mathbf w}$ on \textsc{Image Captioning}, we consider the $21$ experts, the uniform merge, and twenty random merges.
We score their outputs using the inferred utility
$\widehat{\mathbf w}^{\top}\tilde{\mathbf r}$, with the same reward standardization used for that fit.
We select the candidate with the highest mean utility on one half of the evaluation inputs, then compare its utility with that of the \bpm merge on the other half.
This diagnostic measures the inferred reward objective, not preference judgments from the persona.

Across seven personas (including Annotator), six budgets
($10$, $25$, $50$, $100$, $250$, and $500$), and three feedback draws, selection always returns an individual expert.
The selected candidate has higher estimated utility than the \bpm merge in $125$ of $126$ cells.
Thus, directly using inferred reward weights as merge coefficients does not generally maximize the fitted utility over this candidate set.
These utility differences alone do not establish which model's outputs the persona prefers.

\paragraph{Does utility-based selection improve personalization?}\label{app:twostep}
We next compare three ways to use the expert library:
\bpm merges experts using the inferred weights;
\emph{selection} chooses the candidate with the highest mean inferred utility;
and \emph{uniform-weight selection} applies the same selection rule using equal reward weights $\mathbf1/L$.
The last method serves as a control for whether inferred preferences improve candidate selection.

For this experiment, candidates are selected using only the elicitation inputs, without access to evaluation inputs or persona-judge scores.
We then evaluate all three methods against the uniform merge on $451$ \textsc{Summarization} shared test inputs.

Table~\ref{tab:twostep} reports differences between these win rates.
\begin{itemize}
    \item The \emph{BPM $-$ Selection} column compares directly merging experts with selecting a candidate using the same inferred weights.
    \item The \emph{Selection $-$ Uniform-weight selection} column measures the benefit of using inferred rather than equal reward weights for selection.
    \item The \emph{BPM $-$ Uniform-weight selection} column compares the full \bpm procedure with candidate selection that uses no persona-specific feedback.
\end{itemize}

All entries are percentage-point differences; positive values favor the first method named in each column.
They are not direct head-to-head win rates between the two methods named in a column.

On Summarization, direct merging achieves higher win rates
against uniform merging than utility-based selection.
The preceding Captioning diagnostic instead measures fitted
utility under a different selection protocol.
These results support evaluating how inferred weights are used,
but do not establish a utility--preference trade-off
on the same models.

\begin{table}[ht]
\centering\footnotesize
\caption{\textbf{Direct merging and candidate selection on Summarization.} Differences in macro decided win rate against the uniform merge, in percentage points, with $95\%$ paired input-bootstrap intervals. Selection uses inferred reward weights; uniform-weight selection uses equal reward weights. These are differences against a common reference, not direct head-to-head win rates.}
\label{tab:twostep}
\begin{tabular}{lccc}
\toprule
$n$ & \shortstack{BPM $-$\\Selection} & \shortstack{Selection $-$\\Uniform-weight selection} & \shortstack{BPM $-$\\Uniform-weight selection}\\
\midrule
10 & $+39.3\,[34.5,43.8]$ & $+0.5\,[-0.0,1.0]$ & $+39.8\,[35.1,44.3]$\\
25 & $+47.5\,[44.6,50.4]$ & $+0.5\,[-0.2,1.2]$ & $+48.0\,[45.0,51.0]$\\
50 & $+53.6\,[51.4,55.8]$ & $+1.7\,[0.2,3.3]$ & $+55.3\,[53.1,57.5]$\\
100 & $+53.5\,[51.3,55.6]$ & $+2.9\,[1.0,4.9]$ & $+56.4\,[54.4,58.3]$\\
250 & $+45.9\,[43.5,48.3]$ & $+8.0\,[5.9,10.1]$ & $+53.9\,[51.5,56.1]$\\
500 & $+21.8\,[19.9,23.7]$ & $+31.3\,[29.0,33.4]$ & $+53.1\,[50.7,55.4]$\\
\bottomrule
\end{tabular}
\end{table}

\section{Evaluation Reliability and Scoring Sensitivity}\label{app:further}

\subsection{Evaluation Accounting and Scoring Sensitivity}
\label{app:evaluation-accounting}

\paragraph{Personalization gains under supplementary scoring.}
To assess how conditioning on decided judgments affects the reported personalization gains, we supplement the main win rate with a score that retains valid ties and order-inconsistent comparisons.
The central assignment gives these undecided pairs a score of $0.5$; bounds account for the unavailable presentation-order decisions.
Under this assignment, \bpm retains its advantage over uniform merging in all three tasks at both main-table budgets (\Cref{tab:def-pp}). At $n=500$, the supplementary scores against uniform are $74.0\%$, $73.9\%$, and $66.4\%$ on Summarization, Image Captioning, and Story Generation, respectively. 
These results support the usefulness of feedback-derived coefficients under this supplementary treatment of undecided judgments.

\paragraph{Interpreting comparison magnitudes.}
Across the 42 comparisons in \Cref{tab:def-pp}, the central score and the decided win rate agree in direction relative to $50\%$.
The panel comprises uniform merging, the base policy, L2W, RM, BT best-of-eight, DPO, and ICAI on the merge.
Agreement concerns this panel and this central assignment; it does not imply that all admissible assignments give the same conclusion.
Magnitudes can change substantially: on Summarization against uniform at $n=500$, the decided rate is $88.4\%$, the supplementary score is $74.0\%$, and $34.9\%$ of verdicts are order-inconsistent.
Against DPO in the same setting, the supplementary score is $51.0\%$ with interval $[49.8,52.1]$.
The supplementary analysis therefore preserves evidence of useful personalization against uniform merging while clarifying that the headline win rates describe preference conditional on a decided judgment.

\begin{table}[t]
\centering\scriptsize
\caption{\textbf{Scoring sensitivity for seven compared methods.} Dec.: decided win rate. PP: supplementary score with ties and order-inconsistent pairs assigned 0.5; brackets give 95\% input-bootstrap intervals. PP bounds assign 0.25 or 0.75 to order-inconsistent pairs and are not confidence intervals. OI: percentage of order-inconsistent verdicts. Entries use persona-equal aggregation over the available feedback draws. This specified panel contains 42 task--method--budget comparisons.}
\label{tab:def-pp}
\setlength{\tabcolsep}{3pt}
\begin{tabular}{llrccrrccr}
\toprule
 & & \multicolumn{4}{c}{$n=100$} & \multicolumn{4}{c}{$n=500$} \\
\cmidrule(lr){3-6}\cmidrule(lr){7-10}
task & compared method & Dec. & PP [95\% CI] & PP bounds & OI & Dec. & PP [95\% CI] & PP bounds & OI \\
\midrule
\textsc{Summarization} & RS (uniform merge) & 91.7 & 76.5 [75.4, 77.7] & 71.3--81.8 & 21.1 & 88.4 & 74.0 [72.7, 75.3] & 65.3--82.7 & 34.9 \\
 & Base policy & 76.8 & 69.3 [67.8, 70.7] & 62.7--75.9 & 26.4 & 76.7 & 68.5 [67.2, 69.9] & 59.3--77.7 & 36.8 \\
 & L2W & 59.6 & 63.1 [62.2, 63.9] & 59.5--66.6 & 14.0 & 81.9 & 72.5 [71.5, 73.6] & 68.4--76.7 & 16.6 \\
 & Ridge-Merge & 74.9 & 62.8 [61.8, 63.8] & 57.1--68.4 & 22.6 & 84.1 & 67.9 [66.6, 69.2] & 58.9--76.9 & 35.9 \\
 & BT + best-of-8 & 82.5 & 70.0 [68.8, 71.2] & 64.2--75.8 & 23.0 & 85.1 & 69.6 [68.2, 70.9] & 60.4--78.7 & 36.5 \\
 & DPO & 37.9 & 42.4 [41.4, 43.3] & 34.0--50.8 & 33.7 & 57.6 & 51.0 [49.8, 52.1] & 40.7--61.3 & 41.2 \\
 & ICAI-M & 74.4 & 65.0 [64.0, 66.0] & 59.5--70.5 & 22.0 & 88.6 & 72.7 [71.5, 73.9] & 65.0--80.3 & 30.5 \\
\midrule
\textsc{Image Captioning} & RS (uniform merge) & 77.1 & 71.9 [71.0, 72.8] & 68.0--75.8 & 15.5 & 78.7 & 73.9 [73.0, 74.9] & 69.6--78.2 & 17.2 \\
 & Base policy & 74.6 & 71.8 [70.7, 72.9] & 67.9--75.8 & 15.9 & 76.4 & 74.0 [72.8, 75.0] & 69.7--78.2 & 16.9 \\
 & L2W & 72.0 & 67.2 [66.1, 68.2] & 62.1--72.3 & 20.4 & 76.2 & 70.8 [69.9, 71.7] & 65.4--76.1 & 21.4 \\
 & Ridge-Merge & 72.1 & 68.8 [67.9, 69.8] & 64.5--73.2 & 17.5 & 75.3 & 70.9 [69.9, 71.9] & 66.1--75.7 & 19.2 \\
 & BT + best-of-8 & 67.6 & 64.6 [63.2, 65.9] & 57.6--71.5 & 27.8 & 69.4 & 67.5 [66.3, 68.6] & 61.0--73.9 & 25.8 \\
 & DPO & 7.9 & 16.4 [15.5, 17.4] & 11.4--21.5 & 20.2 & 16.0 & 23.5 [22.4, 24.5] & 17.0--29.9 & 25.8 \\
 & ICAI-M & 48.2 & 47.6 [46.4, 49.0] & 41.3--54.0 & 25.5 & 59.5 & 57.8 [56.7, 58.9] & 52.2--63.3 & 22.2 \\
\midrule
\textsc{Story Generation} & RS (uniform merge) & 64.3 & 58.6 [57.4, 59.7] & 49.9--67.3 & 34.8 & 76.1 & 66.4 [65.1, 67.6] & 58.1--74.6 & 32.9 \\
 & Base policy & 71.3 & 62.7 [61.4, 63.9] & 53.7--71.6 & 35.7 & 81.4 & 70.3 [69.1, 71.5] & 62.3--78.2 & 31.8 \\
 & L2W & 58.9 & 55.0 [54.0, 55.9] & 46.0--63.9 & 35.9 & 72.4 & 63.7 [62.7, 64.7] & 55.2--72.1 & 33.8 \\
 & Ridge-Merge & 57.3 & 54.3 [53.6, 55.1] & 45.3--63.3 & 36.0 & 71.5 & 63.3 [62.4, 64.2] & 54.8--71.7 & 33.9 \\
 & BT + best-of-8 & 56.8 & 54.0 [53.0, 55.0] & 44.3--63.7 & 39.0 & 67.1 & 60.2 [59.1, 61.3] & 50.9--69.5 & 37.3 \\
 & DPO & 36.1 & 41.6 [40.3, 42.9] & 33.2--50.0 & 33.6 & 49.0 & 49.6 [48.2, 50.9] & 40.5--58.6 & 36.0 \\
 & ICAI-M & 44.0 & 45.5 [44.7, 46.3] & 36.0--54.9 & 37.8 & 52.6 & 51.9 [51.0, 52.7] & 42.6--61.1 & 37.0 \\
\bottomrule
\end{tabular}
\end{table}

\paragraph{Scoring and aggregation details.}
We classify verdicts as wins $W$, losses $L$, ties $T$, order-inconsistent comparisons $O$, or invalid judgments.
The main metric is $W/(W+L)$. 
Scoring each presentation order as $1$ for a win, $0.5$ for a tie, and $0$ for a loss would assign an order-inconsistent pair $0.25$, $0.5$, or $0.75$, depending on its two decisions.
Those decisions were not separately retained in the verdict records.
With $N=W+L+T+O$, we therefore report
\[
 S_c=\frac{W+0.5T+cO}{N},
 \qquad c\in\{0.25,0.5,0.75\},
\]
where $S_{0.5}$ is the central score and the other assignments bound the possible order-inconsistent contributions.
This score is a supplementary assignment, not an exact reconstruction of the missing order-level decisions.
Invalid judgments are excluded and counted separately; their reported rate is zero throughout the seven-method panel.
Scores are computed per persona and feedback draw, then averaged over draws and equally over personas.
Confidence intervals use 2,000 shared input-bootstrap resamples, holding feedback draws fixed. Assignment bounds and sampling intervals describe different sources of uncertainty.

\subsection{Judge Reliability and Replication}
\label{app:rater}

We examine whether reported preferences persist across repeated judgments, judge families, and paraphrased persona descriptions.
Win rates use decided pairs only: ties and judgments whose preferred response changes with presentation order are excluded.
For the comparison of the base policy with the uniform merge, order-dependent judgments range from $14.4$ to $28.8\%$ across personas in \textsc{Summarization} and $17.0$--$30.0\%$ in \textsc{Image Captioning}.
Pooling over compared methods, budgets, and feedback subsets, order-dependent judgments account for $16.5$ to $36.3\%$ of evaluated inputs.
Ties account for a further $14.1$--$34.5\%$ in \textsc{Summarization} and $4.2$--$31.8\%$ in \textsc{Image Captioning}.
In \textsc{Story Generation}, tie rates are below $5\%$ for five of the seven personas and reach $20.4\%$ for \textit{Librarian}.
These checks assess robustness within simulated evaluation; generalization to a broader population of human users remains untested.

\paragraph{Repeated judgments are consistent.}
We re-evaluate the same $451$ \textsc{Summarization} pairs approximately thirteen hours apart and $200$ fixed \textsc{Image Captioning} pairs per persona in shuffled batch order.
Table~\ref{tab:judge-repeatability} shows high agreement and small changes in win rates across runs.
Repeated evaluation thus largely preserves decided preferences; most variation concerns whether a pair receives a decided judgment.

\begin{table}[htbp]
\centering\small
\setlength{\tabcolsep}{5pt}
\caption{\textbf{Repeatability of persona judgments.}
Cohen's $\kappa$ uses four outcomes: A preferred, B preferred, tie, and order-dependent.
Ranges span personas; decided agreement includes only pairs decided in both runs. Win-rate changes are in percentage points (\%).}
\label{tab:judge-repeatability}
\begin{tabular}{lccc}
\toprule
Task & Cohen's $\kappa$ & Decided agreement & Max.\ win-rate change \\
\midrule
Summarization & $0.90$--$0.97$ & $99.6$--$100\%$ & $1.5$ \\
Image Captioning & $0.92$--$0.99$ & $100\%$ & $2.5$ \\
\bottomrule
\end{tabular}
\end{table}

\paragraph{Preference reversals persist across judge families.}
We hold the personalized models fixed and repeat evaluation with \texttt{gpt-5-mini}, retaining the persona descriptions and two-order protocol.
\textsc{Summarization} is an exception due to data protocol.
This check uses one feedback subset at $n=500$ and covers all six reversal pairs and comparisons with uniform merge, Ridge-Merge, and DPO.
The exception is Story Generation against DPO, where the two estimates lie near 50\%.
The larger change against DPO on Summarization also indicates judge-dependent comparison magnitudes.
Separately, \Cref{fig:analysis}e reports Story Generation results with \texttt{gpt-oss-20b} at $n=100$ and $500$.

\begin{table}[htbp]
\centering\small
\setlength{\tabcolsep}{5pt}
\caption{\textbf{Evaluation with another judge family.}
\bpm's macro decided win rates (\%) with $95\%$CI, using the same generated outputs and one feedback subset at $n=500$.
These estimates are distinct from the three-subset averages in \Cref{tab:main}.}
\label{tab:judge-family}
\begin{tabular}{llcc}
\toprule
Task & Compared method & Original judge & \texttt{gpt-5-mini} \\
\midrule
Summarization & Uniform merge & $90.4 [88.4, 92.3]$ & $92.8 [91.8, 93.9]$ \\
 & Ridge-Merge & $86.1 [84.0, 88.1]$ & $89.3 [87.8, 90.6]$ \\
 & DPO & $55.4 [51.8, 58.8]$ & $83.3 [81.0, 85.6]$ \\
\midrule
Image Captioning & Uniform merge & $79.8 [77.4, 82.1]$ & $74.7 [72.3, 77.1]$ \\
 & Ridge-Merge & $76.3 [74.0, 78.5]$ & $73.6 [71.4, 75.8]$ \\
 & DPO & $15.3 [13.8, 16.9]$ & $15.7 [14.1, 17.4]$ \\
\bottomrule
\end{tabular}
\end{table}

\paragraph{Simulated preferences are sensitive to persona wording.}
With personalized models held fixed, five of six preference
reversals persist after paraphrasing the judge's persona descriptions.
For Ward--ED, paraphrasing the Ward description increases
order-inconsistent judgments from 49.4\% to 73.4\%;
the Ward-personalized model wins 34.5\% of 119 decided comparisons.
Thus, repeatability under fixed prompts does not ensure
robustness to persona wording.

\end{document}